\documentclass[a4paper,fleqn]{cas-dc}

\usepackage[numbers,sort&compress]{natbib}
\usepackage{amsmath,amsfonts}
\usepackage{algorithmic}
\usepackage{algorithm}
\usepackage{array}
\usepackage[caption=false,font=normalsize,labelfont=sf,textfont=sf]{subfig}
\usepackage{textcomp}
\usepackage{bm}
\usepackage{stfloats}
\usepackage{makecell}
\usepackage{url}
\usepackage{verbatim}
\usepackage{graphicx}
\usepackage[normalem]{ulem}
\usepackage{subfig}
\usepackage{amssymb}
\usepackage[dvipsnames]{xcolor}
\usepackage{threeparttable}

\newtheorem{assumption}{Assumption}

\newtheorem{definition}{Definition}
\newtheorem{theorem}{Theorem}
\usepackage{pifont}
\newtheorem{proof}{Proof}
\newtheorem{rem}{Remark}
\usepackage{multirow}
\usepackage{tabularray}
\usepackage{booktabs}
\newcommand{\cmark}{\ding{51}}%
\newcommand{\xmark}{\ding{55}}%
\def\tsc#1{\csdef{#1}{\textsc{\lowercase{#1}}\xspace}}
\tsc{WGM}
\tsc{QE}
\tsc{EP}
\tsc{PMS}
\tsc{BEC}
\tsc{DE}

\begin{document}
\let\WriteBookmarks\relax
\def\floatpagepagefraction{1}
\def\textpagefraction{.001}
\shorttitle{Unified Disturbance Aware Safe Kinematic Control for Closed-Architecture Robots}
\shortauthors{F.Zhang et~al.}

\title [mode = title]{Unified Disturbance Aware Safe Kinematic Control for Closed-Architecture Robots}                      

\author[1,2]{Fan Zhang}
\credit{Conceptualization, Data curation, Formal analysis, Investigation, Methodology, Validation, Visualization, Writing - original draft}
\fnmark[1]

\author[1]{Jinfeng Chen}
\credit{Conceptualization, Formal analysis, Investigation, Methodology, Writing - original draft}
\fnmark[1]

\author[3]{Joseph J. B. Mvogo Ahanda}
\credit{Investigation, Methodology}

\author[4]{Hanz Richter}
\credit{Resources, Writing - review \& editing}

\author[5]{Ge Lv}
\credit{Visualization, Writing – review \& editing}

\author[1,2]{Bin Hu}
\credit{Writing – review \& editing}

\author[1,2]{Qin Lin}
\credit{Conceptualization, Investigation, Methodology, Funding acquisition, Project administration, Supervision, Writing - review \& editing}
\cormark[1]

\affiliation[1]{organization={Department of Engineering Technology, University of Houston},
                country={USA}}

\affiliation[2]{organization={Department of Electrical and Computer Engineering, University of Houston},
                country={USA}}

\affiliation[3]{organization={Department of Biomedical Engineering, The University of Ebolowa},
                country={Cameroon}}

\affiliation[4]{organization={Department of Mechanical Engineering, Cleveland State University},
                country={USA}}

\affiliation[5]{organization={Department of Mechanical Engineering, Clemson University},
                country={USA}}

\cortext[cor1]{Corresponding author}
\fntext[fn1]{These authors contribute equally.}

\begin{abstract}
In commercial robotic systems, it is common to encounter a closed inner-loop torque controller that is not user-modifiable. However, the outer-loop controller, which sends kinematic commands such as position or velocity for the inner-loop controller to track, is typically exposed to users. In this work, we focus on the development of an easily integrated add-on at the outer-loop layer by combining disturbance rejection control and robust control barrier function for high-performance tracking and safe control of the whole dynamic system of an industrial manipulator. This is particularly beneficial when 1) the inner-loop controller is imperfect, unmodifiable, and uncertain; and 2) the dynamic model exhibits significant uncertainty. Stability analysis, formal safety guarantee proof, and hardware experiments with a PUMA robotic manipulator are presented. Our solution demonstrates superior performance in terms of simplicity of implementation, robustness, tracking precision, and safety compared to the state of the art. A demonstration video is available at \url{https://youtu.be/e0palGVU_50}, and is also provided as supplementary material for review.
\end{abstract}




\maketitle
\section*{{Nomenclature}}
\addcontentsline{toc}{section}{Nomenclature}

\noindent
{\begin{tabular}{@{}p{0.18\columnwidth} p{0.76\columnwidth}@{}}
\toprule
\textbf{Symbol} & \textbf{Description} \\
\midrule
$q, \dot{q}, \ddot{q}$ & Actual joint position, velocity, and acceleration \\
$q^\star, \dot{q}^\star, \ddot{q}^\star$ & Desired reference position, velocity, and acceleration \\
$x, u$ & System state vector ($x=[q^\top, \dot{q}^\top]^\top$) and kinematic control input \\
$M, C, G$ & Actual inertia matrix, Coriolis/centrifugal matrix, and gravity vector \\
$\bar{M}, \bar{C}, \bar{G}$ & Nominal inertia matrix, Coriolis/centrifugal matrix, and gravity vector \\
$K_p, K_d$ & Actual proportional and derivative gains of the inner-loop controller \\
$\bar{K}_d$ & Nominal derivative gain of the inner-loop controller \\
$f$ & Actual lumped disturbance\\
$\hat{f}$ & Estimated lumped disturbance\\
$\psi(x), g(x)$ & System drift vector and control input matrix in the nonlinear control-affine model \\
$h(x)$ & Barrier function candidate \\
$\mathcal{C}$ & Safe set\\
$\omega_o$ & Observer bandwidth \\
$\Gamma$ & Disturbance estimation error bound \\
\bottomrule
\end{tabular}}

\section{Introduction}
Robotic systems typically use a hierarchical architecture combining perception, planning, and low-level control, which improves modularity and reliability \cite{niaz2026twist}. In the control layer, a \emph{kinematic controller} (outer loop) and a \emph{dynamic controller} (inner loop) are commonly used. Since tuning the inner loop requires expertise, it is usually hidden from users for safety, forming a \emph{closed architecture} \cite{wang2019dynamic,ahanda2022adaptive,ahanda2022task,khan2024control}. Users can thus only design the kinematic controller, sending position or velocity commands. However, wear, modifications, or payload changes can alter dynamics and degrade inner-loop performance \cite{liu2022advancements}. This work addresses the challenge of achieving unified high-performance and safe tracking control using only the kinematic controller, despite an imperfect, unmodifiable inner loop, and uncertain dynamics.


{This hierarchical perspective is conceptually closely related to admittance control, where a high-level regulation objective is systematically converted into a desired motion command for a lower-level controller to track. Such outer-loop command generation structures have been validated in physical human-robot interaction \cite{keemink2018admittance} and wearable robotic systems \cite{cao2025load}, transforming high-level objectives into actionable kinematic commands. However, standard admittance control typically assumes perfect inner-loop tracking and designs the control law to render the system as a desired mass-spring-damper model. In contrast, our framework specifically addresses the practical challenge of an imperfect, unmodifiable inner-loop controller. By actively estimating and compensating for lumped disturbances, our approach forces the degraded closed-architecture system to behave as a nominal second-order inertia model, ensuring trajectory tracking stability and safe control.}

Existing works typically rely on the following assumptions: 1) The inner-loop controller is known and accessible to the user; 2) The inner-loop controller is known but not accessible to the user; 3) The inner-loop controller has perfect tracking for outer-loop commands. Our goal is to minimize reliance on knowledge of the dynamic model and the inner-loop controller, while still guaranteeing tracking and safe control using only an outer-loop kinematic controller.
{A more comprehensive comparison with these state-of-the-art methods is presented in Section \ref{sec: Literature Review} and Table \ref{tab:related-work}.}

As shown in Fig.~\ref{fig:Feedback Planning Framework}, existing robotic systems typically send desired position ($q^*$) or velocity ($\dot{q}^*$) commands to the inner-loop controller. Our solution, inserted between them (gray box), comprises a disturbance-rejection tracking controller, an extended state observer (ESO), and an extended state observer-based robust quadratic programming (ESOR-QP) module. The ESO estimates the total disturbance $\hat{f}$ using the joint angle $q$ and the safe kinematic control {input $u_{\text{safe}}$}. This estimate is compensated in the tracking controller to produce nominal commands {$u_{\text{total}}$}, which are monitored by the ESOR-QP. The disturbance estimate, together with its associated error bound $\Gamma$, is incorporated as a robust constraint to ensure safety in the presence of disturbances.

The contributions of our work are summarized as follows:
\begin{itemize}
 \item We establish a disturbance-driven representation for closed-architecture robotic systems using only output measurements, enabling consistent disturbance estimation despite unknown inner-loop dynamics. This representation serves as a shared interface that enables a co-design of tracking and safety.

\item We develop a robust CBF formulation that directly incorporates disturbance estimates in the original system dynamics, in contrast to DOB-CBF approaches that operate on transformed CBF dynamics \cite{dacs2022robust,alan2023disturbance}. The proposed method relies only on output measurements, avoiding the full-state assumption commonly adopted in DOB-based methods \cite{wang2023disturbance}. Moreover, explicit bounds on disturbance estimation errors are derived directly from the observer error dynamics. This avoids the need for additional Lyapunov-based approximation or bounding procedures as in \cite{zhang2024eso}. The framework further extends safety guarantees to systems with higher relative degree, beyond the relative-degree-one setting considered in \cite{chen2023robust}.
 


\item The proposed method is implemented as a plug-and-play outer-loop module that requires no modification of the inner-loop controller, and is validated in real time (1 kHz) on a PUMA 500 manipulator.

\end{itemize}

\begin{figure}[htbp]
    \centering
    \includegraphics[width=1\linewidth]{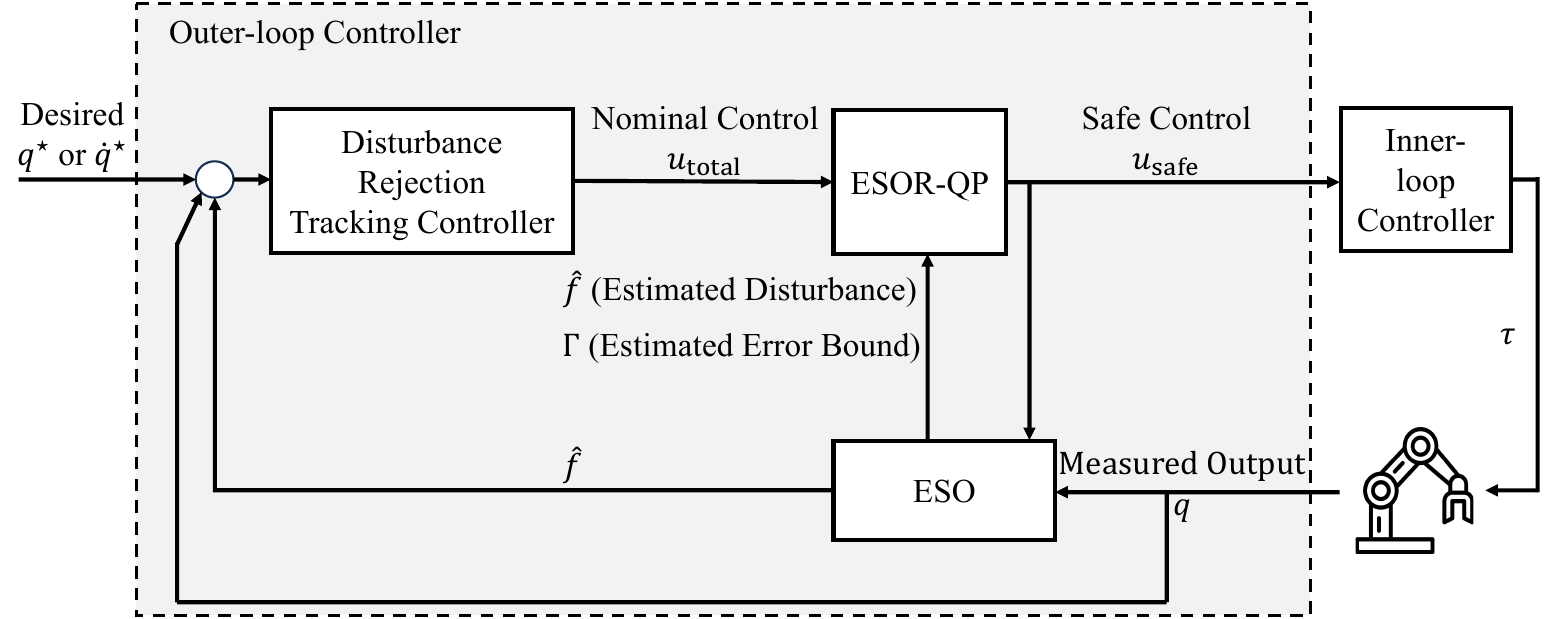}
    \caption{The proposed framework in this paper, is marked as the Outer-loop Controller. The total
estimated disturbance $\hat{f}$ is used for both disturbance compensation and safe control using ESOR-QP. The disturbance estimation error bound $\Gamma$ is used for ESOR-QP for robust safe control under disturbance.}
    \label{fig:Feedback Planning Framework}
\end{figure}

The rest of the paper is organized as follows: Section \ref{sec: Literature Review} reviews the related work. Section \ref{sec: Problem Statement} introduces the preliminaries and problem formulation. Section \ref{sec:Feedback planning} presents the proposed framework. Hardware experiments are reported in Section \ref{sec: Hardware Experiments}. Concluding remarks are in Section \ref{Conslusion and future work}.

\section{Related Work}
\label{sec: Literature Review}
An overview of the comparison between our work and the most related works is presented in Table \ref{tab:related-work}.

\subsection{Adaptive Control for Closed Architecture}

\textbf{Learning-based Adaptive Control:} \cite{ahanda2022adaptive,ahanda2022task,khan2024control} employ neural network–based outer-loop adaptation to approximate the kinematics, dynamics, and inner-loop controller. As summarized in Table~\ref{tab:related-work}, they make the weakest assumption about the inner loop, identical to ours, where $\Psi(\cdot)$ denotes an unknown structure possibly containing proportional, integral, or state-dependent terms. However, they rely on numerous tuning parameters and limited interpretability.

\textbf{Conventional Adaptive Control:} \cite{wang2019dynamic} introduces dynamic modularity with an adaptive controller under a moderate assumption of a known PI/PID inner-loop structure.

These studies address tracking control only, without safety considerations. In our experiments, we compare tracking performance with the learning-based adaptive control \cite{ahanda2022adaptive} due to the shared inner-loop assumption.

\begin{table*}[htbp]
\centering
\begin{threeparttable}
    \caption{Overview of our work's position in the literature}
    \label{tab:related-work}
    \begin{tabular}{ccccc}
        \toprule
        Methods & Tracking & Safety & \makecell{Assumption of \\ Inner-loop Controller} & Ref.\\
        \midrule
        \makecell{Learning-based\\ Adaptive Control} & \cmark  & \xmark  & \makecell{Weak\\ {$\tau = -K_d(\dot q-\dot q_d)+\Psi\!\left(q,\dot q,q_d,\int_0^t q(s)\,ds,\int_0^t q_d(s)\,ds\right)$}} & \cite{ahanda2022adaptive,ahanda2022task,khan2024control} \\
        \hline
        \makecell{Conventional\\ Adaptive Control} & \cmark & \xmark & \makecell{Moderate, known PID/PI structure} & \cite{wang2019dynamic} \\
        \hline
        Reduced-order CBF & - & \cmark & \makecell{Strong, known PD structure,\\ perfect tracking or\\ bounded error} & \cite{molnar2021model,singletary2021safety,molnar2023safety} \\ 
        \hline
        Robust CBF & - & \cmark & - & \cite{mrdjan2018robust, kolathaya2019input, alan2021safe, buch2022robust, nguyen2022robust, alan2023parameterized, wang2023disturbance, dacs2022robust, alan2023disturbance, zhao2020adaptive, sun2024safety, das2025robust, wang2024safety, cao2024safety, zhang2024eso, wang2023composite}\\
        \hline
        Ours & \cmark & \cmark & \makecell{Weak\\ {$\tau = -K_d(\dot q-\dot q_d)+\Psi\!\left(q,\dot q,q_d,\int_0^t q(s)\,ds,\int_0^t q_d(s)\,ds\right)$}} & - \\
        \bottomrule
        
    \end{tabular}
    \begin{tablenotes}
        \footnotesize
        \item[] {The two dashes for CBFs’ tracking indicate that we only compare our safety performance with them, as their nominal tracking controllers vary significantly across papers. The third dash for the robust CBF indicates that those works do not consider two-loop controllers.}
    \end{tablenotes}
\end{threeparttable}
\end{table*}

\subsection{Reduced-order Model-based CBF}
Many CBF-based safety approaches rely on high-fidelity dynamic models with torque control, yet obtaining such models remains challenging. Reduced-order or kinematic CBFs have been introduced to ensure safety with lower-fidelity models. For instance, \cite{molnar2021model} guarantees safety via safe velocity control but assumes a perfect inner-loop controller capable of instant tracking. Similarly, \cite{singletary2021safety,molnar2023safety} employ kinetic energy-based CBFs that offer partial robustness to model uncertainty but still depend on numerous dynamic parameters and a known PD-type inner loop. In contrast, our work addresses closed-architecture robots with minimal knowledge of the inner-loop controller and significant model uncertainty.

\subsection{Robust CBF}

Robust CBFs have recently been developed to handle model uncertainties and external disturbances \cite{mrdjan2018robust, kolathaya2019input, alan2021safe, buch2022robust, nguyen2022robust, alan2023parameterized, wang2023disturbance, alan2023disturbance, sun2024safety, das2025robust, wang2024safety, chen2023robust, zhao2020adaptive, dacs2022robust, zhang2024eso, wang2023composite, cao2024safety}. Among them, disturbance observer–based robust CBFs (DOB-CBFs) \cite{chen2023robust,zhao2020adaptive,dacs2022robust,zhang2024eso,wang2023composite,cao2024safety,wang2023disturbance,alan2023disturbance,sun2024safety,das2025robust,wang2024safety} enable active disturbance estimation without requiring worst-case bounds. Unlike these methods where estimation error bounds are absent or loosely derived via Lyapunov analysis, our approach obtains them directly from the observer error dynamics, resulting in tighter and less conservative bounds. Experimental comparisons with state-of-the-art methods (\emph{e.g.}, \cite{dacs2022robust}) further demonstrate this advantage. Recent work has explored ESO-based robust CBF design under output measurements \cite{chen2023robust}, focusing on safety for relative-degree-one systems. In contrast, this work considers safe tracking for robotic systems with higher relative degrees, where tracking and safety are addressed simultaneously.

\section{Preliminaries and Problem Formulation}
\label{sec: Problem Statement}
{\textbf{Notation Convention:} Throughout this paper, the superscript $^\star$ denotes the reference trajectory, while the subscript $_d$ denotes the kinematic command sent to the inner loop. For mathematical rigor, we unify the kinematic control command as the input vector $u$. A bar over a variable ($\bar{\cdot}$) represents a nominal parameter, and a hat ($\hat{\cdot}$) indicates an estimated value.}

\subsection{System Dynamics}
We consider a general robotic system with an Euler-Lagrange formulation incorporating a generalized external disturbance force $F_{ext}$: 
\begin{equation}
\begin{cases}

\label{eq:EL}
{M}\ddot{q}+{C}\dot{q}+{G}+F_r=\tau+ {J}^{T}F_{ext}\\
 y=\zeta (q)
\end{cases},
\end{equation}
where $q, \dot{q}, \ddot{q} \in \mathbb{R}^n$ are the joint position, velocity, and acceleration, respectively. ${M} \in \mathbb{R}^{n\times n}$ is the symmetric positive-definite inertia matrix; ${C}\in \mathbb{R}^{n\times n}$ denotes Coriolis and centrifugal term; ${G} \in \mathbb{R}^{n}$ is the gravity term; $F_r \in \mathbb{R}^{n}$ accounts for the friction torque; $\tau \in \mathbb{R}^{n}$ represents the torque (provided by the inner-loop controller); ${J} \in \mathbb{R}^{6\times n}$ is the Jacobian matrix; and $F_{ext} \in \mathbb{R}^{6}$ represents a generalized external disturbance force. The output  $y \in \mathbb{R}^{m}$, and $\zeta:\mathbb{R}^{n} \to\ \mathbb{R}^{m}$ is the mapping from joint space to task space.

The most common inner-loop controller for robot manipulators typically consists of a PID controller or PD controller, often combined with dynamic compensation, such as gravity. For stability, most inner-loop controllers include a derivative term for damping, $(\dot{q}-\dot{q}_d)$, where $\dot{q}_d$ is a joint velocity command.

\begin{assumption}
   We assume the inner-loop controller has the following structure:
\begin{equation}
\label{eq:inner}
{\tau = -K_d(\dot q-\dot q_d)
+\Psi\!\left(q,\dot q,q_d,
\int_0^t q(s)\,ds,
\int_0^t q_d(s)\,ds\right)},
\end{equation}
with $K_d$ a derivative control gain, and $\Psi(\cdot)$ is a general, unknown function that may include proportional, integral, derivative, and dynamic compensation terms. To the best of our knowledge, this assumption is the most general in the literature (see Table \ref{tab:related-work}), requiring minimal knowledge of the inner-loop controller, similar to \cite{ahanda2022adaptive}, and more general than \cite{wang2019dynamic}.
\end{assumption}

\begin{assumption}
\label{assump}
    The kinematics of the system are assumed to be known for control in task space.
\end{assumption}

The system dynamic model combining \eqref{eq:EL} and \eqref{eq:inner} is:
\begin{equation}
\label{eq:SECOND}
\begin{split}
    \ddot{q}  &=  {M}^{-1} \Bigl(-{C} \dot{q} -{G}-F_r+{J}^{T}F_{ext}-{K}_d \Dot{q}+\Psi(\cdot)  + {K}_d \dot{q}_d\Bigr)
    \end{split}.
\end{equation}
We have removed torque control, as it is not accessible in a closed architecture. {To explicitly build the control problem from an input-output viewpoint, we define the system state as $x \triangleq [q^\top,\dot q^\top]^\top \in \mathcal{X} \subset \mathbb{R}^{2n}$ and the kinematic control command manipulated in the outer-loop layer as the control input $u \triangleq \dot{q}_d \in \mathcal{U} \subset \mathbb{R}^n$ from the outset.}  Another reason for choosing $\dot{q}_d$ instead of $q_d$ as the control input is to avoid differentiation-induced noise.
\subsection{Robot Kinematic Control with Known Model}
Neglecting $F_r$ and $F_{ext}$ in \eqref{eq:EL}, the exact system dynamics is used to design the control law:
\begin{equation}
\label{eq: control law}
\begin{split}
    {u_{\text{ideal}}}  &= {K}_d^{-1} \Bigl( {M}u_0 
    +{C} \dot{q} +{G}+{K}_d \Dot{q}-\Psi(\cdot)  \Bigr)
    \end{split},
\end{equation}
where the position control input $u_0$ is defined as follows:
\begin{equation}
\label{eq: control types}
\begin{split}
u_0 &=\ddot{q}^\star+k_p(q^\star-q)+k_d(\dot{q}^\star-\dot{q})\\
    \end{split},
\end{equation}
where $\ddot{q}^\star, \dot{q}^\star, q^\star \in \mathbb{R}^{n}$ represent the desired acceleration, velocity, and position of a reference joint trajectory generated by an upstream planner. $k_{p}\in \mathbb{R}^{n \times n}$ and $k_{d}\in \mathbb{R}^{n \times n}$ are the proportional and derivative gain matrices, respectively. $\ddot{q}^\star$ serves only as a feedforward term, while $u_0$ is a position  controller. By substituting \eqref{eq: control law} and \eqref{eq: control types} into \eqref{eq:SECOND}, we get: 
\begin{equation}
\label{eq: final system}
    \ddot{q}  = \ddot{q}^\star+k_p(q^\star-q)+k_d(\dot{q}^\star-\dot{q}).
\end{equation}
Proper $k_p$ and $k_d$ make ${q}$ converge to ${q}^\star$ asymptotically.

\subsection{Control Barrier Functions}
The system \eqref{eq:SECOND} can be rewritten as a nonlinear control-affine system when ignoring the disturbance vector \(F_{ext}\):
\begin{equation}\label{eq:affine}
    \dot{x}=\psi(x)+g(x)u,
\end{equation}
where $x \triangleq [q,\dot{q}]^T \in \mathcal{X} \subset \mathbb{R}^{2n}$, $\psi : \mathbb{R}^{2n} \rightarrow \mathbb{R}^{2n}$ and $g : \mathbb{R}^{2n} \rightarrow \mathbb{R}^{2n \times n}$ are Lipschitz continuous, and {$u \in \mathcal{U} \subset \mathbb{R}^n$ is the} control input vector.

The safety of system \eqref{eq:affine} can be guaranteed using a safety set. 
A set $\mathcal{C}$ is considered to be a safety set if it is forward invariant in the state space $\mathcal{X}$, \emph{i.e.}, for system \eqref{eq:affine} if solutions for some $u \in \mathcal{U}$ starting at any initial safe condition $x(0)  \in \mathcal{C}$ satisfy $x(t)  \in \mathcal{C}$, $\forall t \ge 0$.
The safety set $\mathcal{C}$ is defined as a 0-superlevel set of a continuously-differential function \(h(x): \mathbb{R}^{2n} \rightarrow \mathbb{R}\) as:

\begin{equation}
    \mathcal{C} = \{x \in \mathbb{R}^{2n}~|~h(x) \ge 0\}.
\end{equation}

The function \(h\) is used to synthesize a controller with safety guarantees via a CBF. 

\begin{definition}
    (\emph{Exponential CBF (ECBF)} \cite{nguyen2016exponential, xiao2019control}) Consider system \eqref{eq:affine} with relative degree \(r\) for an \(r\)-times continuously differentiable function \(h\), \emph{i.e.}, \(L_gL_{\psi}h(x)=\cdots=L_gL_{\psi}^{r-2}h(x)=0\) and \(L_gL_{\psi}^{r-1}h(x)\neq 0, \forall x\in \mathcal{C}\). \(h(x)\) is an ECBF if there exists a row vector \(K_a\in\mathbb{R}^r\) satisfying \(\forall x\in\mathcal{C}\)
    \begin{equation}
            \sup\limits_{u \in \mathcal{U}} \big( L_{\psi}^r h(x)  +L_g L_{\psi}^{r-1}h(x)u \big) \geq -K_a \eta_b(x), 
    \end{equation}
\end{definition}
where \(\eta_b(x)=[h(x), \dot{h}(x), \cdots, h^{(r-1)}(x)]^T\), \(K_a=[k_1, \cdots, k_r]\), and the values of \(k_1, \cdots, k_r\) satisfy specific properties given in \cite{nguyen2016exponential, xiao2019control}. 

To guarantee safety, the control problem is formulated as a quadratic program (QP) with a CBF as a hard constraint \cite{ames2016control}:

\begin{equation}
\label{eq:QP-CBF}
\begin{split}
    u^* (x) & = \operatorname*{arg\,min}_{u\in U} \|u-k(x)\|^2 \\
    \text{s.t.} 
    \quad L_{\psi}^r h(x) & +L_g L_{\psi}^{r-1}h(x)u \geq -K_a \eta_b(x),
\end{split}
\end{equation}
where \(k(x)\) is a nominal control law.

\textbf{Research Objective:} Given the system in \eqref{eq:SECOND}, the nominal dynamics with limited knowledge of the system and inner-loop controller are defined as  
\begin{equation}\label{eq:nominal model}
\ddot{q} = \bar{M}^{-1}\!\left(-\bar{C}\dot{q} - \bar{G} + \bar{K}_d {u}\right),
\end{equation}
where $\bar{M}$, $\bar{C}$, and $\bar{G}$ are nominal parameters, and $\bar{K}_d$ is the nominal inner-loop gain.  
The goal is to design a {$u$} controller that achieves accurate trajectory tracking and task-space safety (\emph{e.g.}, collision avoidance) under an unknown inner-loop controller, uncertain dynamics, and external disturbances.

\begin{rem}
 $\Bar{C}$ and $\Bar{G}$ can be a zero matrix and a zero vector, respectively, with minimal knowledge of the system. We use $\Bar{C}$ and $\Bar{G}$ here to make the formulation more general. The model \eqref{eq:nominal model} can be further reduced to:
\end{rem}
 \begin{equation}
 \label{eq:reduced nominal model}
     \ddot{q} = \Bar{M}^{-1} \Bar{K}_d {u}.
 \end{equation}
We use the nominal model \eqref{eq:nominal model} for general mathematical derivation, whereas the reduced nominal model \eqref{eq:reduced nominal model} is employed in hardware experiments to evaluate the tracking performance and safety under limited information.

\section{Proposed Framework}
\label{sec:Feedback planning}

\subsection{Robot Kinematic Control with Nominal Model}

The control law in \eqref{eq: control law} is ideal, assuming known dynamics and an inner-loop controller. To handle uncertainties, we propose the following practical law based on nominal dynamics:

\begin{equation}
\label{eq:nominal control law}
\begin{split}
    {u_{\text{nominal}}}  &= \Bar{K}_d^{-1} \Bigl( \Bar{M}u_0 
    +\Bar{C} \dot{q} +\Bar{G}  \Bigr).
    \end{split}
\end{equation}

The following part shows the dynamics of the closed-loop system with our nominal control law \eqref{eq:nominal control law}. We define the parametric discrepancies as:
\begin{equation}
\label{eq: difference}
\begin{split}
    \Delta K_d &\triangleq K_d\Bar{K}_d^{-1},\\
    \Delta {M} &\triangleq \Delta K_d\bar{M}-{M},\\
    \Delta{C} &\triangleq \Delta K_d\bar{C}-{C},\\
    \Delta {G} &\triangleq \Delta K_d\Bar{G}- {G}.
    \end{split}
\end{equation}

By rearranging \eqref{eq:SECOND}, we get:
\begin{equation}
\label{eq: derive1}
\begin{split}
&{M}\ddot{q}+{C}\dot{q}+{G}+F_r\\
    &=-K_d (\dot{q}-{u})+\Psi(\cdot)+ {J}^{T}F_{ext}. \\
    \end{split}
\end{equation}

Substituting {$u=u_{\text{nominal}}$} and \eqref{eq: difference} into \eqref{eq: derive1}, the right-hand side of \eqref{eq: derive1} becomes:

\begin{equation}
\label{eq: derive right hand}
\begin{split}
     &-K_d \dot{q}+K_d\Bar{K}_d^{-1} \Bigl( \Bar{M}u_0
    +\Bar{C} \dot{q} +\Bar{G}  \Bigr)+\Psi(\cdot)+ {J}^{T}F_{ext}\\
    &=-K_d \dot{q}+\Delta K_d \Bigl( \Bar{M}u_0
    +\Bar{C} \dot{q} +\Bar{G}  \Bigr)+\Psi(\cdot)+ {J}^{T}F_{ext}.\\
    \end{split}
\end{equation}

The left-hand side of  \eqref{eq: derive1} becomes:
\begin{equation}
\label{eq: derive left hand}
\begin{split}
    &\Bigl( \Delta K_d\bar{M}-\Delta {M}\Bigr)\ddot{q}+\Bigl(\Delta K_d\bar{C}-\Delta{C}\Bigr)\dot{q}\\
    &+\Bigl( \Delta K_d\Bar{G}-\Delta {G}\Bigr)+F_r.
    \end{split}
\end{equation}

The resulting closed-loop system is:
\begin{equation}
\label{eq: nominal final system}
\begin{split}
    \ddot{q}  = \ddot{q}^\star+k_p(q^\star-q)+k_d(\dot{q}^\star-\dot{q})+f.
    \end{split}
\end{equation}
where 
\begin{equation}
\label{eq:f}
\begin{split}
    f=(\Delta K_d\bar{M})^{-1}\Bigl(&\Delta {M}\ddot{q}+\Delta{C}\dot{q}+\Delta {G}- F_r\\
    &+\Psi(\cdot)+{J}^{T}F_{ext}-K_d \dot{q}  \Bigr).
\end{split}
\end{equation}

Let $f$ denote the total disturbance, including internal uncertainties and external disturbances. Compared with \eqref{eq: final system}, $f$ degrades system performance. To counter this, an ESO \cite{chen2022relationship} is employed to estimate $f$ for disturbance rejection.

\subsection{Extended State Observer Design}\label{Sec:ESO}

For the \(i\)-th joint of the robotic manipulator \eqref{eq:nominal model}, we consider the following subsystem:
\begin{equation}\label{eq:ESO1}
    \begin{cases}
    \begin{array}{l}
         \dot{x}_{1i} = x_{2i}  \\
         \dot{x}_{2i} = F_i(x)+G_i(x){u_i}+f_i,
    \end{array}
    \end{cases}
\end{equation}
where \(x_{1i}=q_i\); {\(F_i(x)\), \(G_i(x)\), and \(f_i\) are the \(i\)-th row of \(F=\Bar{M}^{-1} (-\Bar{C} \dot{q} -\Bar{G}) \in \mathbb{R}^n\), \(G=\Bar{M}^{-1}\Bar{K}_d  \in \mathbb{R}^n\), and \(f\), respectively.}

By treating \(f_i\) as an extended state, the augmented system is given by:
\begin{equation}
\label{eq:augmented state space}
\begin{cases}
  \begin{array}{l}
\dot{x}_{1i}=x_{2i} \\
\dot{x}_{2i}=F_i(x)+G_i(x){u_i}+x_{3i}\\
\dot{x}_{3i}=\dot{f}_i.
  \end{array}
\end{cases}
\end{equation}

In this work, all model terms are evaluated using observer states to ensure implementability and improved noise attenuation. Specifically, let \(\hat{x}_i = [\hat{x}_{1i}, \hat{x}_{2i}, \hat{x}_{3i}]^\top\) denote the observer state. Then, a third-order ESO is designed as:
\begin{equation}
\label{eq: ESO}
\begin{cases}
  \begin{array}{l}
\dot{\hat{x}}_{1i}=\hat{x}_{2i}+\beta_{1i}(x_{1i}-\hat{x}_{1i})\\
\dot{\hat{x}}_{2i}=F_i(\hat{x})+G_i(\hat{x}){u_i}+\hat{x}_{3i}+\beta_{2i}(x_{1i}-\hat{x}_{1i})\\
\dot{\hat{x}}_{3i}=\beta_{3i}(x_{1i}-\hat{x}_{1i}),
  \end{array}
\end{cases}
\end{equation}
where \(\hat{x}_{1i}\), \(\hat{x}_{2i}\), and \(\hat{x}_{3i}\) denote the estimates of position, velocity, and disturbance, respectively.

Define the estimation error as
\[
\vartheta_i=
\begin{bmatrix}
x_{1i}-\hat{x}_{1i}\\
x_{2i}-\hat{x}_{2i}\\
x_{3i}-\hat{x}_{3i}
\end{bmatrix}.
\]

Due to the use of estimated states in \(F_i(\cdot)\) and \(G_i(\cdot)\), a mismatch arises between the true and observer dynamics. In this work, such mismatch is lumped into the disturbance term, leading to an equivalent augmented disturbance that remains bounded with bounded derivative. Under this standard lumped-disturbance treatment, the estimation error dynamics can be expressed as:
\begin{equation}\label{eq: error dynamic}
\dot{\vartheta}_i = A_i \vartheta_i + E_i \dot{f}_i,
\end{equation}
where
\[
A_i=
\begin{bmatrix}
-\beta_{1i} & 1 & 0\\
-\beta_{2i} & 0 & 1\\
-\beta_{3i} & 0 & 0
\end{bmatrix},\quad
E_i=
\begin{bmatrix}
0\\
0\\
1
\end{bmatrix}.
\]

The observer gains are selected such that the eigenvalues of \(A_i\) are placed at \(-\omega_{o_i}\), i.e., the characteristic polynomial is \((s+\omega_{o_i})^3\). Accordingly, the gains are chosen as \(\beta_{1i}=3\omega_{o_i}\), \(\beta_{2i}=3\omega_{o_i}^2\), and \(\beta_{3i}=\omega_{o_i}^3\), where \(\omega_{o_i}>0\) denotes the observer bandwidth \cite{gao2003scaling}. In the subsequent controller and CBF design, \(x\) denotes the system state for analysis, while its estimated value \(\hat{x}\) is used in implementation; the resulting mismatch is absorbed into the lumped disturbance

\subsection{Controller Design}
$\hat{f}$ can be used to develop a disturbance rejection controller that complements the control law \eqref{eq:nominal control law}:
\begin{equation}
\label{eq:overall control law}
\begin{split}
    &{u_{\text{total}} = {u_{\text{nominal}}} + u_{f}} ,\\
    &{u_{\text{nominal}}}  = \Bar{K}_d^{-1} \Bigl( \Bar{M}u_0 
    +\Bar{C} \dot{q} +\Bar{G}  \Bigr) ,\\
    &{u_{f}}=-\Bar{K}_d^{-1}\bar{M} \hat{f} ,
    \end{split}
\end{equation}
where {$u_{\text{nominal}}$} is the nominal control from \eqref{eq:nominal control law}, {$u_{f}$} rejects disturbance, and they are summed to get {the total control input $u_{\text{total}}$}. 


\subsection{ Controller Stability Analysis}
\begin{assumption}\label{asm1}
There exists a positive known constant \(l_f\) such that for any \(x\in \mathcal{X}\), \(u\in \mathcal{U}\), and \(t\geq0\), the following inequality holds:
    \begin{equation}\label{eq_6}
        \left|\dfrac{\partial f_i(t,x,u)}{\partial t}\right|\leq l_f. 
    \end{equation}
\end{assumption}
Assumption \ref{asm1} implies that \(f\) is Lipschitz continuous with respect to time, and $\dot{f}$ is bounded. 
\begin{theorem}
\label{lemma}
Given the robotic system \eqref{eq:SECOND} with kinematic control input, the ESO \eqref{eq: ESO} with appropriate observer bandwidth \(\omega_{o_i}\) for each joint, and the nominal control law \eqref{eq:overall control law} with appropriate tuned \(k_p\) and \(k_d\), {the closed-loop system is uniformly ultimately bounded (UUB). Specifically, the tracking error $\epsilon$ converges to an ultimate bound that is inversely proportional to the observer bandwidth \(\omega_{o_i}\).}

\end{theorem}
\begin{proof}
Since $A_i$ is Hurwitz, we have  the following decomposition \cite[Chapter 1.3]{perko2013differential}: 
\begin{equation}\label{eq: decomposition}
e^{A_it}=P_ie^{\Lambda_i t}P_i^{-1},
\end{equation}
where $P_i$ is a matrix whose columns are the eigenvectors of 
$A_i$, and ${\Lambda_i}$ is a diagonal matrix with the corresponding eigenvalues in the diagonal elements.

Therefore, there exists a constant $c_{1_i}$ such that \cite[Chapter 1.9]{perko2013differential}
\begin{equation}\label{eq: inequality}
\left\lVert e^{A_it}\right\rVert \leq \left\lVert P_i\right\rVert \left\lVert e^{\Lambda _it}\right\rVert\left\lVert P_i^{-1}\right\rVert =c_{1_i}\left\lVert e^{\Lambda_i t}\right\rVert \leq c_{1_i}\left\lVert e^{a_it}\right\rVert, \end{equation}
where {$a_i$} is the $\lambda_{max} (\Lambda_i)$, the maximum eigenvalue of $\Lambda_i$, which, in our case, is $-\omega_{o_i}$.
 
The solution to the error dynamics in \eqref{eq: error dynamic} is given by:
\begin{equation}\label{eq: solution}
\vartheta_i = e^{A_i(t-t_0)}\vartheta_i(t_0) + \int_{t_0}^{t} e^{A_i(t-\tau)}E_i\dot{f}_i(\tau) \, d\tau,
\end{equation}
where $t_0$ is the initial time. Substituting \eqref{eq: inequality} into \eqref{eq: solution}, we obtain:
\begin{equation}\label{eq: solution relax}
\begin{array}{ll}
\left\lVert \vartheta_i \right\rVert \ &\leq c_{1_i}\left\lVert e^{-\omega_{o_i}(t-t_0)}\right\rVert \left\lVert \vartheta_i(t_0)\right\rVert\\  
&+ \int_{t_0}^{t} c_{1_i}\left\lVert e^{-\omega_{o_i}(t-\tau)}\right\rVert \left\lVert E_i\dot{f}_i(\tau)\right\rVert \, d\tau\\
&\leq c_{1_i}\left\lVert e^{-\omega_{o_i}(t-t_0)}\right\rVert \left\lVert \vartheta_i(t_0)\right\rVert\\
&+ \frac{c_{1_i}}{\omega_{o_i}} \sup_{\tau\in (t_0, t)} \left\lVert E_i\dot{f}_i(\tau)\right\rVert \ .\\
\end{array}
\end{equation}
Since $\dot{f}_i$ is bounded, $e^{-\omega_{o_i}(t-t_0)}$ decays exponentially, the estimation error $\left\lVert \vartheta_i \right\rVert \ \leq \sigma_i$ is bounded. 

Thus, the disturbance estimation error satisfies:
\begin{equation}\label{eq: estimation error}
    \left\lVert\Tilde{f}\right\rVert=\left\lVert f-\hat{f} \right\rVert \leq \sigma.
\end{equation}

For the $i$-th joint subsystem, substituting the control law \eqref{eq:overall control law} into \eqref{eq:SECOND} results in the following \(i\)-th closed-loop subsystem dynamics:
\begin{equation}\label{eq: error function}
    \ddot{q}_i  = \ddot{q}_i^\star+k_{p_i}(q_i^\star-q_i)+k_{d_i}(\dot{q}_i^\star-\dot{q}_i)+\Tilde{f}_i ,
\end{equation}
where $\Tilde{f}_i=f_i-\hat{f}_i$ represents the disturbance estimation error. The error dynamic can be expressed as:
\begin{equation}\label{eq: control error dynamic}
    \dot{\epsilon}=H\epsilon+ \xi,
\end{equation}
where $\epsilon=\begin{bmatrix}
        q_i-q_i^\star \\ \dot{q}_i-\dot{q}_i^\star\\
    \end{bmatrix}$, $H=\begin{bmatrix}
        0& 1 \\ -k_{p_i}& -k_{d_i}\\
    \end{bmatrix}$, and $\xi=\begin{bmatrix}
        0 \\ \Tilde{f}_i\\
    \end{bmatrix}$.

The matrix $H$ can be designed to be Hurwitz. Choose $V=\frac{1}{2}\epsilon^TQ\epsilon$ as the Lyapunov function candidate. Its derivative along \eqref{eq: control error dynamic} is given by:
\begin{equation}\label{eq: V dot}
\dot{V} = \tfrac{1}{2}(\dot{\epsilon}^T Q \epsilon + \epsilon^T Q \dot{\epsilon}) = -\|\epsilon\|^2 + \xi^T Q \epsilon.
\end{equation}
By applying the Cauchy–Schwarz inequality, we obtain:
\begin{equation}\label{eq: V dot ieq}
\begin{array}{ll}
\dot{V}& \leq -\left\lVert \epsilon \right\rVert^2+ \left\lVert \xi \right\rVert \left\lVert Q \right\rVert \left\lVert\epsilon \right\rVert.\\
\end{array}
\end{equation}
Given that 
$Q$ is positive definite, $\left\lVert Q \right\rVert \leq \lambda_{max} (Q)$, and the $\xi$ is bounded by $\sigma$,  the inequality \eqref{eq: V dot ieq} becomes:
\begin{equation}\label{eq: V dot ieq futher}
\begin{array}{ll}
\dot{V}& \leq -\left\lVert \epsilon \right\rVert^2+ \sigma \lambda_{max} (Q)\left\lVert\epsilon \right\rVert.
\end{array}
\end{equation}
Next, using the inequality:
\begin{equation}
\begin{array}{ll}
\sigma \lambda_{max} (Q)\left\lVert\epsilon \right\rVert \leq \frac{1}{2}\left\lVert \epsilon \right\rVert^2+ \frac{1}{2}\sigma^2 \lambda^2_{max} (Q),
\end{array}
\end{equation}
\eqref{eq: V dot ieq futher} can be rewritten as:
\begin{equation}\label{eq: V dot ieq final}
\begin{array}{ll}
\dot{V}& \leq -\frac{1}{2}\left\lVert \epsilon \right\rVert^2+ \frac{1}{2}\sigma^2 \lambda^2_{max} (Q)\\
& \leq -\frac{V}{\lambda_{max} (Q)} +\frac{1}{2}\sigma^2 \lambda^2_{max} (Q).
\end{array}
\end{equation}
Solving the inequality \eqref{eq: V dot ieq final} yields:
\begin{equation}\label{eq: V dot solve}
\begin{array}{ll}
{V}& \leq (v(0)-\frac{\kappa}{\eta})e^{-\eta t}+\frac{\kappa}{\eta},
\end{array}
\end{equation} 
where $\eta=\frac{1}{\lambda_{max} (Q)}$ and $\kappa=\frac{1}{2}\sigma^2 \lambda^2_{max} (Q)$.
{From \eqref{eq: V dot solve}, as $t \to \infty$, the Lyapunov function $V$ is bounded by $\frac{\kappa}{\eta}$. This implies that the tracking error $\epsilon$ is uniformly ultimately bounded. In practical implementations, selecting a sufficiently large finite bandwidth ensures that the disturbance estimation error $\sigma$ remains small, thereby bounding the closed-loop tracking error. \hfill $\square$}
\end{proof}

\subsection{Robust High-order CBF for Safe Control}\label{ESOR-QP}
 
By adding the disturbance term, system \eqref{eq:nominal model} becomes:
\begin{equation}\label{eq:RCBF_dynamics}
    \dot{x}=\psi(x)+g(x)u+\bar{g}f,
\end{equation}
where \(f=[f_1, \cdots, f_n]^T \in \mathcal{F} \subset \mathbb{R}^n\) is a total disturbance vector in each input channel, \(\bar{g}=[\mathbf{0}_{n\times n}, \mathbf{I}_{n\times n}]^T\), \(\mathbf{0}_{n\times n}\) and \(\mathbf{I}_{n\times n}\) are \(n\times n\) zero and identity matrices, respectively. Note that \(u\) and \(f\) are in the same channel, \emph{i.e.}, the \(n\times n\) matrix of the upper part of \(g(x)\) is a zero matrix and each row in the lower part of \(g(x)\) includes an independent input, so \(f\) is a matched disturbance vector. 

To define safety sets for system \eqref{eq:RCBF_dynamics}, we consider an \(r\)-times continuously differentiable function \(h(x)\), where the relative degrees of \(h(x)\) with respect to both \(u\) and \(f\) are \(r\), given that \(f\) is matched.
A series of functions are defined as follows:
\begin{equation}
\label{eq:generalcbf}
\begin{array}{ll}
h_0(x)= & h(x),\\
h_1(x)= & \dot{h}_0(x)+\gamma_1 h_0(x),\\
 & \vdots \\
h_r(x)= & \dot{h}_{r-1}(x)+\gamma_{r} h_{r-1}(x),
\end{array}
\end{equation}
where \(\gamma_1, \cdots, \gamma_r\) are positive constants. The corresponding series of safety sets are
\begin{equation}\label{eq:generalSafetySets}
    \begin{array}{ll}
        \mathcal{C}_0= & \{x\in\mathbb{R}^{2n}: h_0(x)\geq 0\}, \\
        \mathcal{C}_1= & \{x\in\mathbb{R}^{2n}: h_1(x)\geq 0\}, \\
         & \vdots \\
        \mathcal{C}_r= & \{x\in\mathbb{R}^{2n}: h_r(x)\geq 0\}.
    \end{array}
\end{equation}


\begin{theorem}\label{Th_RCBFwithf}
    For system \eqref{eq:RCBF_dynamics}, the relative degrees of \(h(x)\) with respect to both the input \(u\) and the disturbance \(f\) are \(r\), and \(f\) is known. If the initial states satisfy \(x(0)\in \mathcal{C}_s= \cap_{j=0}^r\mathcal{C}_j\), then any Lipschitz continuous controller \(u(x)\in K_{\text{cbf}}(t,x,f)\) renders the set \(\mathcal{C}_s\) forward invariant for system \eqref{eq:RCBF_dynamics}, where 
    \begin{equation}\label{eq:RCBFUSetF}
        \begin{array}{r@{}l}
            K_{\text{cbf}}(t,x,f)\triangleq & \{ u\in\mathcal{U}: L_{\psi}^{r}h(x) + L_g L_{\psi}^{r-1} h(x)u\\
            & + L_{\bar{g}} L_{\psi}^{r-1} h(x)f + \sum\limits_{j=0}^{r-1}k_{j} h^{(j)}(x) \geq 0\}, 
        \end{array}
    \end{equation}
   where \(h^{(j)}(x)\) is the \(j\)th time derivative of \(h(x)\), and \(k_{j}, j=0, \cdots, r-1\), are the coefficients of polynomial \(s^r+k_{r-1}s^{r-1}+\cdots+k_{0}\) with roots at \(-\gamma_1, \cdots, -\gamma_r\).
\end{theorem}

\begin{proof}
    From \eqref{eq:generalcbf}, \(h_r\) can be written as
    \begin{equation}\label{eq:h_r}
        \begin{array}{r@{}l}
            h_r(x)= & (s+\gamma_r)(s+\gamma_{r-1})\cdots(s+\gamma_1)h(x) \\
            = & s^r h(x) + k_{r-1}s^{r-1}h(x) + \cdots + k_0 h(x),
        \end{array}
    \end{equation}
    where \(s\triangleq \frac{d}{dt}\). Since the control input \(u\) and disturbance \(f\) both have the same relative degree of \(r\), \(u\) and \(f\) do not explicitly show in \(h^{(j)}(x), 0\leq j\leq r-1\) until 
    \begin{equation}\label{eq:h_derivative}
        h^{(r)}(x)=L_{\psi}^{r}h(x) + L_g L_{\psi}^{r-1} h(x)u + L_{\bar{g}} L_{\psi}^{r-1} h(x)f. 
    \end{equation}
    By substituting \eqref{eq:h_derivative} into \eqref{eq:h_r} and comparing it with \eqref{eq:RCBFUSetF}, the control values in \(K_{\text{cbf}}(t,x,f)\) guarantee \(h_r(x)\geq 0\), \(\forall t>0\). From the last equation in \eqref{eq:generalcbf}, we have \(\dot{h}_{r-1}(x)+\gamma_r h_{r-1}(x)\geq 0\). Then \(\dot{h}_{r-1}(x)\geq 0\) for any \(x\in\partial \mathcal{C}_{r-1}\). According to Nagumo's theorem \cite{blanchini2008set}, since \(x(0)\in \mathcal{C}_{r-1}\), we have \(h_{r-1}(x)\geq 0\), \(\forall t>0\). Iteratively, \(\dot{h}_0(x)+\gamma_1 h_0(x)\geq 0\), and we have \(h_0(x)\geq 0\), \(\forall t>0\), as \(x(0)\in \mathcal{C}_0\). Therefore, the set \(\mathcal{C}_s\) is forward invariant for system \eqref{eq:RCBF_dynamics}.\hfill $\square$
\end{proof}

However, the disturbance vector \(f\) is not available in practice to enforce the inequality in \eqref{eq:RCBFUSetF}. The estimated disturbance vector \(\hat{f}\) obtained in Subsection \ref{Sec:ESO} is used to devise our robust CBF. Then, \eqref{eq:RCBFUSetF} can be reformulated as: 
\begin{equation}\label{eq:RCBFUSetFhat}
    \begin{array}{r@{}l}
        K_{\text{cbf}}(t,x,f,\hat{f}) \triangleq & \{u\in \mathcal{U}: L_{\psi}^r h(x) + L_g L_{\psi}^{r-1} h(x)u \\
         & + L_{\bar{g}} L_{\psi}^{r-1} h(x)\hat{f} - L_{\bar{g}} L_{\psi}^{r-1} h(x)(\hat{f}-f) \\
         & +\sum\limits_{j=0}^{r-1}k_j h^{(j)}(x) \geq 0\}. 
    \end{array}
\end{equation}
Since the disturbance estimation error \(\hat{f}-f\) is unknown in \eqref{eq:RCBFUSetFhat}, the following assumption is made for the disturbance estimation error bound of ESO. 


The estimation error of \(\hat{f}_i\) is as follows \cite{chen2022relationship}:
\begin{equation}\label{eq_18}
    f_i(k)-\hat{f}_i(k)=p(k)* \Delta f_i(k),
\end{equation}
where \(*\) represents convolution and 
\begin{equation}\label{eq_19}
\footnotesize
    \begin{array}{l}
         p(k)=   \\
          \begin{cases}
                1 & 1\leq k \leq r_i+1\\
                \sum\limits_{\iota = 1}^{r_i+1}  \dfrac{1}{(\iota-1)!} \left(\prod\limits_{j=-\iota+1}^{-1}(k+j)\right)(1-\omega_{o_i})^{\iota-1}\omega_{o_i}^{k-\iota} & k\ge r_i+2.
            \end{cases} 
    \end{array}
\end{equation}
For simplicity, let \(\prod_{j=0}^{-1}(k+j)=1\). Note that \eqref{eq_18} is the disturbance estimation error formulated in the discrete-time domain. \(f_i(k)=f_i(kT_s)\), \(\hat{f}_i(k)=\hat{f}_i(kT_s)\), \(\Delta f_i(k)=f_i((k+1)T_s)-f_i(kT_s)\), \(T_s\) is the sample time used in calculating the error bound, and \(r_i\) and \(\omega_{o_i}\) are the relative degree of \(\dot{q}_i\) (output measurement) with respect to \(f_i\) (disturbance as input) in the \(i\)-th joint subsystem and the poles of the \(i\)-th ESO in the discrete-time domain, respectively. Readers are referred to our previous work \cite{chen2022relationship} for a complete proof of this error bound.

From the definition of derivative, Assumption \ref{asm1}, and \eqref{eq_18}, the disturbance estimation error bound of ESO is 
\begin{equation}\label{eq:RCBF_ErrorBound}
    |f_i(k)-\hat{f}_i(k)|\leq\Gamma_i(\omega_{o_i}, T_s)=\left(\sum\limits_{k=1}^{\infty}p(k)\right)l_f T_s.
\end{equation}
{\begin{rem}
As mentioned in \cite{chen2023robust}, although \eqref{eq:RCBF_ErrorBound} is derived in the discrete-time domain, the disturbance estimation error bound obtained from it remains consistent in the continuous-time domain. According to \eqref{eq:RCBF_ErrorBound}, the error bound is directly proportional to the sampling time $T_s$. Therefore, $T_s$ should be sufficiently small to yield a tight and effective safety bound. In our experiment setting, $T_s$ is set to $1$ ms (corresponding to the $1$ kHz hardware control loop), which is small enough to guarantee safety without introducing excessive conservatism. \(\omega_{o_i}\) in the discrete-time domain is converted from the continuous-time domain through the Z-transform.
\end{rem}}

\begin{theorem}\label{Th_RCBFwithFhat}
    Given the system \eqref{eq:RCBF_dynamics} and the ESO in \eqref{eq: ESO} under Assumption \ref{asm1}, any controller \(u(x)\in K_{\text{rcbf}}\) renders the set \(\mathcal{C}_s\) forward invariant for system \eqref{eq:RCBF_dynamics}, where
    \begin{equation}\label{eq:RCBFUSetFbound}
        \begin{array}{r@{}l}
            K_{\text{rcbf}}(t,x,f,\hat{f}) \triangleq & \{u\in \mathcal{U}: L_{\psi}^r h(x) + L_g L_{\psi}^{r-1} h(x)u \\
             & + L_{\bar{g}} L_{\psi}^{r-1} h(x)\hat{f} - |L_{\bar{g}} L_{\psi}^{r-1} h(x)|\Gamma(\omega_o, T_s) \\
             & +\sum\limits_{j=0}^{r-1}k_j h^{(j)}(x) \geq 0\}, 
        \end{array}
    \end{equation}
    \(|L_{\bar{g}} L_{\psi}^{r-1} h(x)|\) denotes the absolute value of each element, and \(\Gamma(\omega_o, T_s)=[\Gamma_1(\omega_{o_1},T_s), \cdots, \Gamma_n(\omega_{o_n},T_s)]^T\). 
\end{theorem}

\begin{proof}
     From Theorem \ref{Th_RCBFwithf}, we need to prove \(h_r(x)\geq 0\), \(\forall t>0\). Substituting \eqref{eq:h_derivative} and \eqref{eq:RCBF_ErrorBound} into \eqref{eq:h_r} yields
    \begin{equation}
        \begin{array}{r@{}l}
            h_r(x) = &  L_{\psi}^{r}h(x) + L_g L_{\psi}^{r-1} h(x)u + L_{\bar{g}} L_{\psi}^{r-1} h(x)f \\
             & + \sum\limits_{j=0}^{r-1}k_j h^{(j)}(x) \\
            = &  L_{\psi}^{r}h(x) + L_g L_{\psi}^{r-1} h(x)u + L_{\bar{g}} L_{\psi}^{r-1} h(x)\hat{f} \\
              & -L_{\bar{g}} L_{\psi}^{r-1} h(x)(\hat{f}-f)+\sum\limits_{j=0}^{r-1}k_j h^{(j)}(x) \\
            \geq & L_{\psi}^{r}h(x) + L_g L_{\psi}^{r-1} h(x)u + L_{\bar{g}} L_{\psi}^{r-1} h(x)\hat{f} \\
              & -|L_{\bar{g}} L_{\psi}^{r-1} h(x)|\Gamma(\omega_o,T_s)+\sum\limits_{j=0}^{r-1}k_j h^{(j)}(x) \geq 0.
        \end{array}
    \end{equation}

    Then, we have \(h_r(x)\geq 0\), \(\forall t>0\) as \(x(0)\in\mathcal{C}_r\). Using the same procedure in Theorem \ref{Th_RCBFwithf}, the set \(\mathcal{C}_s\) is forward invariant for system \eqref{eq:RCBF_dynamics}. \hfill $\square$
\end{proof}

The safety specification requires the end-effector to maintain a non-negative distance from a virtual wall. 
The safety function is defined as \( h(x) = y - y_0 \),  where \( y \) and \( y_0 \) denote the positions of the end-effector and the virtual wall, respectively. It can be verified that \(h(x)\) has a relative degree of two \cite{nguyen2016exponential}: 
\begin{equation}
\label{eq:wallcbf}
\begin{array}{ll}
h_0=h ,\\
h_1=\dot{h}_0+\gamma h_0&=\dot{y}+\gamma(y-y_0),\\
h_2=\dot{h}_1+\gamma h_1&=\ddot{y}+2\gamma\dot{y}+\gamma^2(y-y_0),\\
\end{array}
\end{equation}
where a positive constant \(\gamma_1=\gamma_2=\gamma\) is chosen, 
\begin{equation*}
\dot{y}=\left(\frac{\partial  \zeta}{\partial q} \right) \dot{q}, ~\text{and}~
\ddot{y}=\frac{d}{dt} \left(\frac{\partial  \zeta}{\partial q} \right) \dot{q}+\frac{\partial \zeta}{\partial q} \ddot{q}. 
\end{equation*}
Combining equations \eqref{eq:nominal model}, \eqref{eq:f}, and \eqref{eq:wallcbf} yields
\begin{equation}
\label{eq:h2_derive}
\begin{array}{ll}
h_2=&2\gamma\dot{y}+\gamma^2(y-y_0)+\frac{d}{dt} \left(\frac{\partial  \zeta}{\partial q} \right) \dot{q}\\
&+\frac{\partial  \zeta}{\partial q} \Bigl(- \bar{M}^{-1} (\bar{C} \dot{q} +\bar{G}-\bar{K}_d {u}) +f \Bigr).
\end{array}
\end{equation}

To simultaneously achieve tracking and safe control, the nominal controller for {$u_{\text{total}}$}, designed in~\eqref{eq:overall control law}, should be minimally intervened. Thus, the following QP-CBF is constructed: 
\begin{equation}
\begin{aligned}
\label{eq:qp}
{u_{\text{safe}}} = &\operatorname*{arg\,min}_{u\in \mathcal{U}} \|  u-{u_{\text{total}}}\|^2 \\
\textrm{s.t.} \quad & h_i(x(0))>0, ~i=0, 1, 2\\
& {L_{\psi}^2h(x)} + L_g L_{\psi} h(x)u \\
& + L_{\bar{g}} L_{\psi} h(x)\hat{f} - {|L_{\bar{g}} L_{\psi} h(x)|}\Gamma(\omega_o, T_s) \\
& +2\gamma\dot{h}(x)+\gamma^2 h(x) \geq 0.
\end{aligned}
\end{equation}

\section{Hardware Experiments}
\label{sec: Hardware Experiments}
We use the PUMA 500 robot for hardware experiments, as shown in Fig. \ref{fig:PUMA 500 robot}. It has three main joints and a spherical wrist, providing six degrees of freedom. $q_1$, $q_2$, and $q_3$ correspond to the waist, shoulder, and elbow joints, while $q_4$, $q_5$, and $q_6$ represent the wrist rotation, bend, and flange angles. Since our focus is task-space tracking (\emph{i.e.}, wrist center position), only the first three joints are used. Each joint is driven by a DC motor powered by an Adept Technology MV-19 Power Chassis. The control program is built in MATLAB Simulink, which generates C code and an SDF file, and runs via a dSPACE 1103 for control and data acquisition.

\begin{figure}[htbp]
    \centering
    \includegraphics[width=2in]{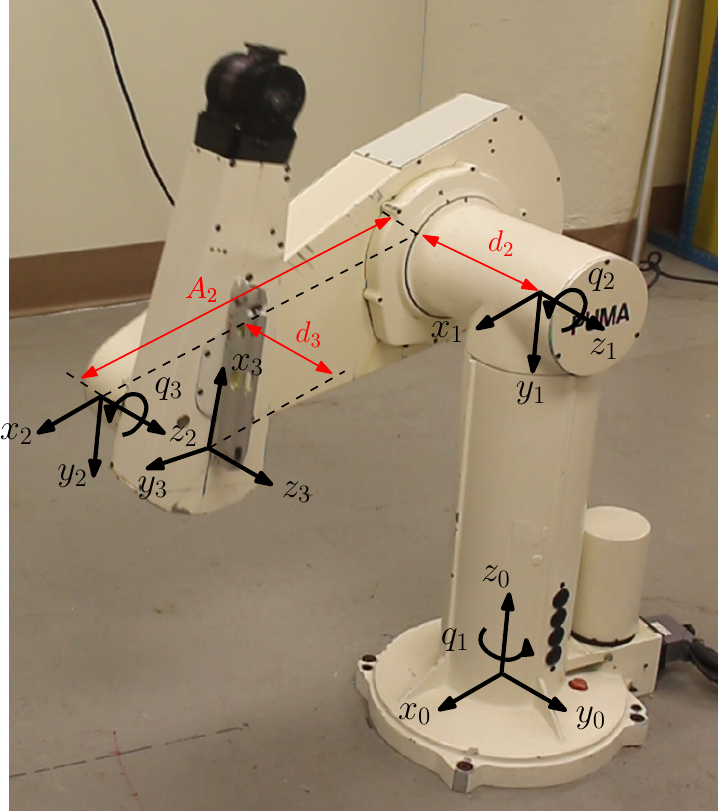}
    \caption{PUMA 500 robot used as a testbed in our research.}
    \label{fig:PUMA 500 robot}
\end{figure}

The inner-loop controller is configured as $\tau = -K_p (q-q_d)-K_d (\dot{q}-\dot{q}_d)$, where $K_{p}=\mathrm{diag}([18, 144, 9])$ and $K_d=\mathrm{diag}([18, 12, 9])$. The gains are determined by the actuator driver and voltage-to-torque conversion, and are not retuned for the specific task.

This setup is used to emulate a degraded inner-loop performance scenario, as commonly encountered in practice due to factors such as wear, payload variations, and unmodeled dynamics, as discussed in the Introduction. The structure and true parameters of the inner-loop controller remain unknown to the outer-loop controller, consistent with the closed-architecture assumption.

The nominal second-order system model \eqref{eq:reduced nominal model} is used for controller design. The reference trajectory is defined as 
${q}^*(t)=\big(0.5\sin t,\; 0.25\sin(2t)+0.534,\; 0.25\sin(2t)+\tfrac{\pi}{2}\big)^\top$.

The tracking performance of the inner-loop controller alone is illustrated in Fig.~\ref{fig: nominalPD}. As shown in Fig.~\ref{fig: nominalPDperformance} and Fig.~\ref{fig:nominalPDtrajectory}, the inner-loop controller exhibits non-ideal tracking behavior under this setting. This reflects realistic conditions where the inner-loop controller is fixed and not optimized for the given task.

\begin{rem}
The inner-loop controller is not intentionally weakened. Instead, this configuration is designed to emulate realistic performance degradation and to evaluate the proposed method under multiple sources of uncertainty, including inner-loop imperfections, model mismatch, and external disturbances.
\end{rem}

In the following, we demonstrate that the proposed method significantly improves both tracking performance and safety under these uncertainties.

We use the nominal model \eqref{eq:reduced nominal model} throughout the experiments. A key parameter in the controller design is the control gain in \eqref{eq:overall control law}, $b_0 \triangleq \bar{M}^{-1}\bar{K}_d$, which depends on the nominal inertia matrix $\bar{M}$ and the inner-loop proportional gain. The nominal inertia $\bar{M}$ is adopted from \cite{khalaf2019trajectory}. Following a common ESO tuning rule, we tune each joint individually by initially setting $b_0$ large to ensure stability and then gradually reducing it to improve tracking accuracy. The final nominal control gain is $\mathrm{diag}([20, 40, 10])$. The observer and control bandwidths are set to $80~\mathrm{rad/s}$ and $10~\mathrm{rad/s}$, respectively, for each joint. We refer readers to \cite{xue2015performance} for additional instructions and theoretical support on tuning nominal control gain.

\begin{figure*}[h!]
\centering
\subfloat[\textnormal{Joint space performance}]{\includegraphics[width=2.9in]{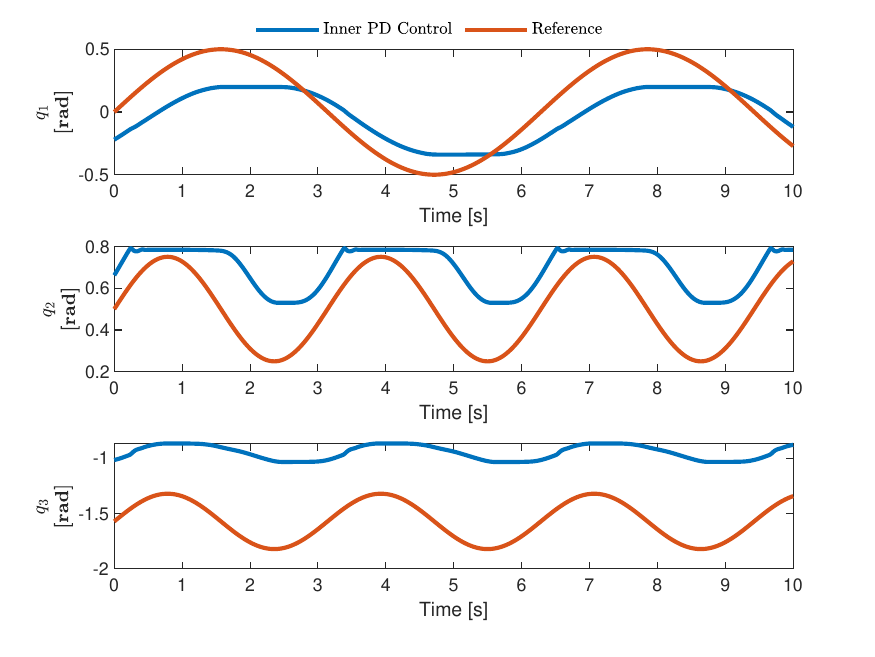}%
\label{fig:  nominalPDperformance}}
\hfil
\subfloat[\textnormal{Task space performance}]{\includegraphics[width=2.9in]{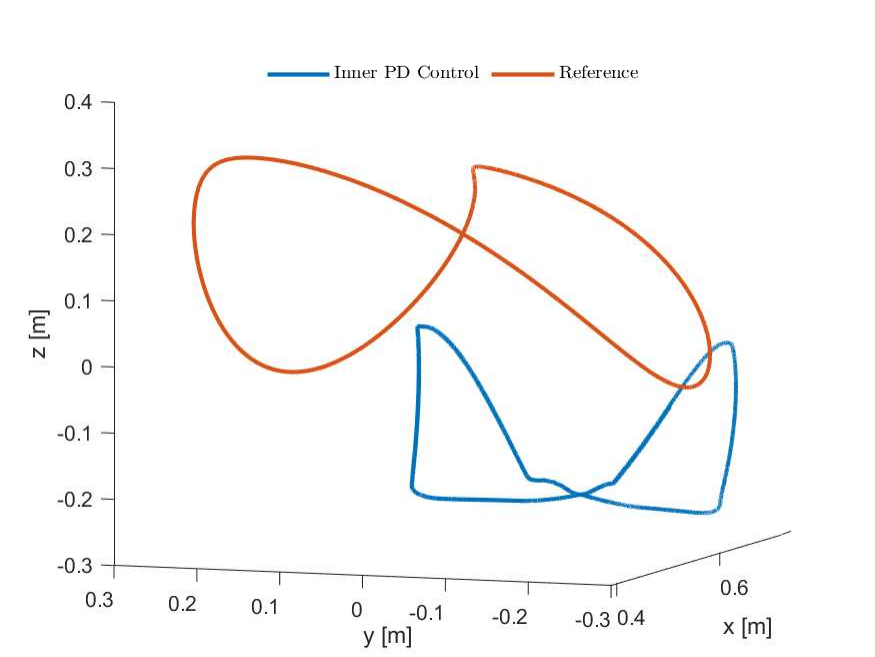}%
\label{fig:nominalPDtrajectory}}
\caption{Inner-loop PD controller tracking performance over $10$ s under degraded conditions. (a) Trajectory tracking performance in the joint space. (b) Trajectory tracking performance in the Cartesian space. The inner-loop controller can not track the reference trajectory.}
\label{fig: nominalPD}
\end{figure*}

\subsection{Tracking Control}
We compare the tracking performance of ours, with \cite{ahanda2022adaptive}. The reason is that we share the same minimal assumptions about the inner-loop controller, which, to the best of our knowledge, are the most general in the literature.

\subsubsection{\textbf{Scenario 1: Normal Case without Payload}}
The robot tracks a reference trajectory without a payload. The tracking performance of the learning-based adaptive controller \cite{ahanda2022adaptive} and ours is shown in Fig.~\ref{fig: Tracking S1}. Fig.~\ref{fig: TrackingS1joint} compares the reference (dashed) and actual joint trajectories of $q_1$, $q_2$, and $q_3$, where blue and red lines represent our method and \cite{ahanda2022adaptive}, respectively. The x-axis denotes time (s), and the y-axis denotes joint angles (rad). Fig.~\ref{fig: TrackingS1task} shows the Cartesian trajectories in $x$, $y$, and $z$. The close alignment of the final and reference trajectories indicates that both controllers achieve accurate tracking despite model uncertainties and an imperfect inner loop.

The proposed method achieves tracking performance comparable to the learning-based adaptive controller in the no-payload experiment. As shown in Fig.~\ref{fig:Tracking S1 inner command}, both approaches yield similarly smooth control signals. Transient responses occur initially, particularly in the second joint (Fig.~\ref{fig: TrackingS1joint}), due to the offset between the initial and reference positions. The learning-based controller exhibits a longer transient phase (see Figs.~\ref{fig: TrackingS1joint} and \ref{fig:Tracking S1 inner command}). This reflects a typical trade-off in adaptive control between transient duration and tracking accuracy.
\begin{figure*}[htbp]
\centering
\subfloat[\textnormal{Joint space comparison}]{\includegraphics[width=2.9in]{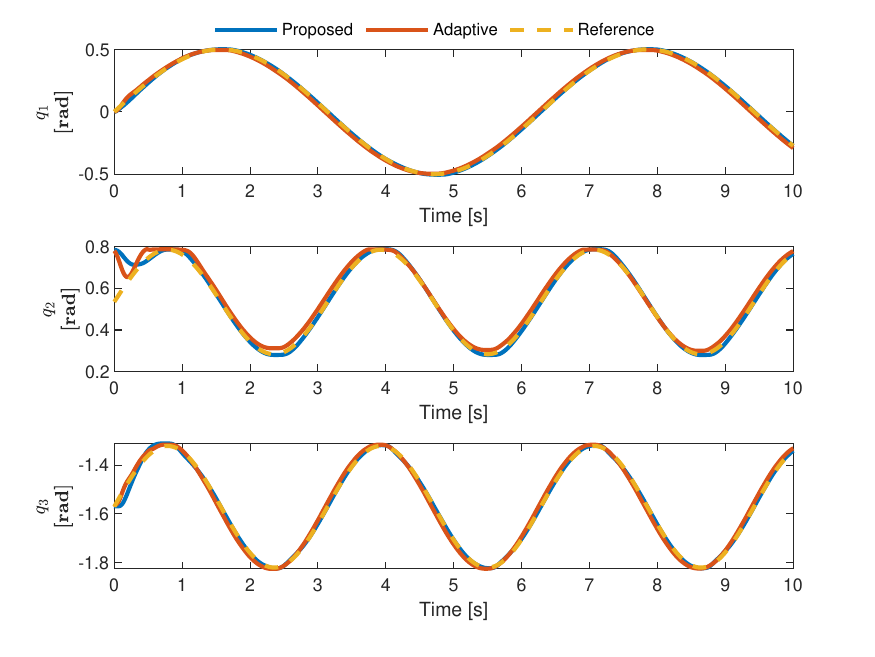}%
\label{fig:  TrackingS1joint}}
\hfil
\subfloat[\textnormal{Task space comparison}]{\includegraphics[width=2.9in]{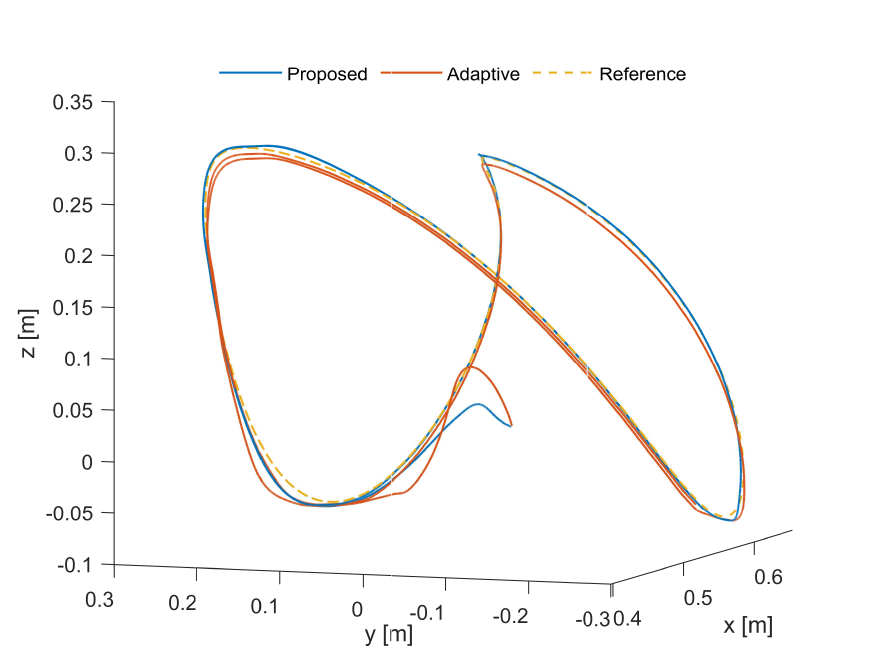}%
\label{fig:  TrackingS1task}}
\caption{Comparative tracking precision analysis: (a) Root Mean squared error (RMSE) in joint space: Proposed ($4.1 \times 10^{-3}$, $2.76 \times 10^{-2}$, $7.6 \times 10^{-3}$) radians vs. \cite{ahanda2022adaptive} ($4.3 \times 10^{-3}$, $2.7 \times 10^{-2}$, $5.7 \times 10^{-3}$) radians for each joint. (b) Cartesian tracking error RMSE: Proposed ($2.3 $ mm, $2.4 $ mm, $15.5$ mm) vs. \cite{ahanda2022adaptive} ($2.9$ mm, $2.1$ mm, $16.1$ mm) for each joint. Data sampled at 1 kHz over a 10 s trajectory execution.}
\label{fig: Tracking S1}
\end{figure*}

\begin{figure}[htbp]
    \centering
    \includegraphics[width=2.9in]{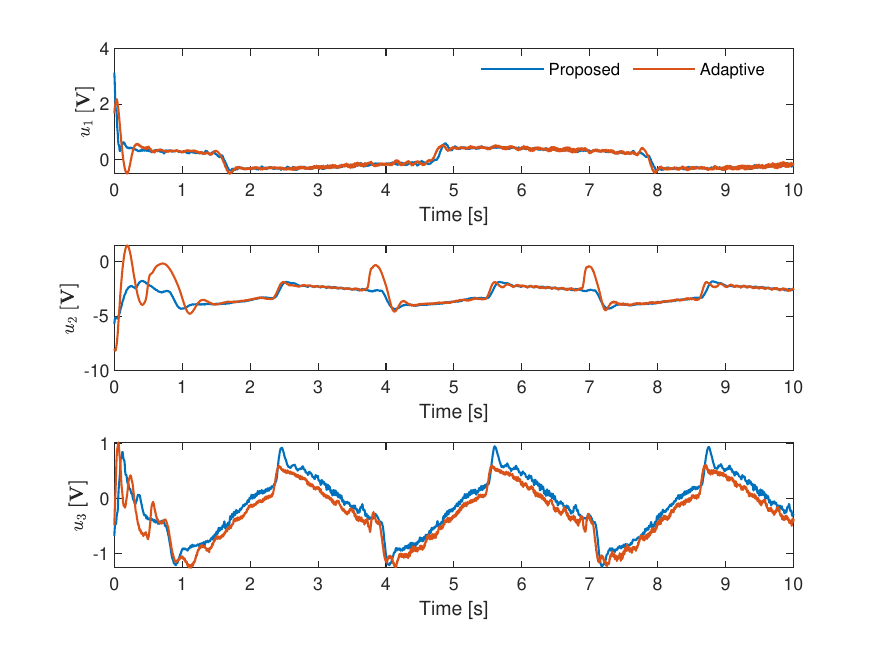}
    \caption{Inner-loop control command comparison for $10$ s. $u_1$, $u_2$, and $u_3$ are the control inputs (voltages) for joint $q_1$, $q_2$, and $q_3$, respectively.}
    \label{fig:Tracking S1 inner command}
\end{figure}

\subsubsection{\textbf{Scenario 2: Robustness Testing with Different Payloads}}
We evaluate controller robustness with payloads of 1 kg, 1.5 kg, and 2 kg attached to the end-effector. Each payload experiment is repeated five times for ANOVA analysis, with identical controller and observer parameters. Fig.~\ref{fig:Tracking Error Box Plot} and Fig.~\ref{fig:Adaptive Tracking Error Box Plot} present the tracking error ANOVA results for both methods. As shown earlier, our approach achieves a shorter transient time than the learning-based adaptive controller. For clarity, transient data (first 5 s) are excluded, and the plots show steady-state tracking errors (after 5 s).

\begin{figure*}[htbp]
  \centering
  \subfloat[\textnormal{ANOVA analysis for ours}]{\includegraphics[width=0.8\linewidth]{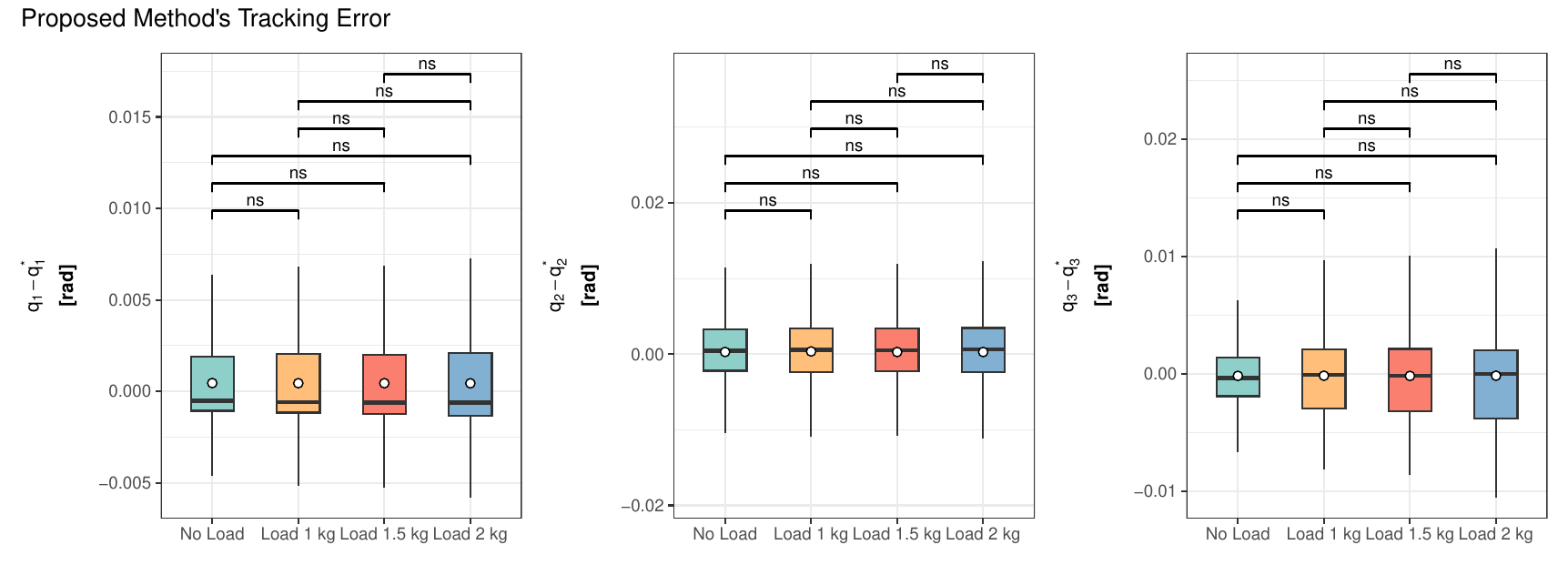}%
  \label{fig:Tracking Error Box Plot}}\\
  
  \subfloat[\textnormal{ANOVA analysis for learning-based adaptive control \cite{ahanda2022adaptive}}]{\includegraphics[width=0.8\linewidth]{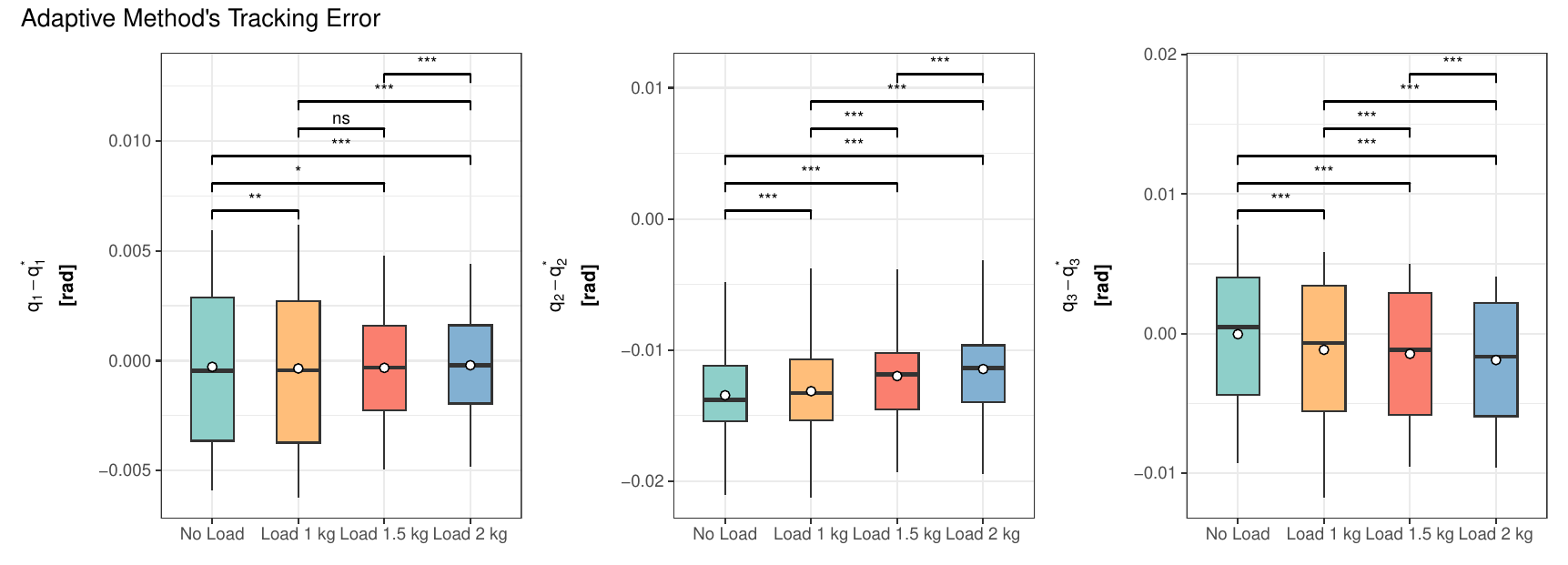}%
  \label{fig:Adaptive Tracking Error Box Plot}}

  \caption[Steady-state tracking error comparison under different payloads]{Steady-state tracking error comparison under different payloads. (a) ANOVA of ours. (b) ANOVA of the learning-based adaptive control \cite{ahanda2022adaptive}. The proposed method shows consistent tracking across payloads, as indicated by uniform box plots, while the adaptive method exhibits larger error variations, implying reduced robustness. Significance levels: ns ($p > 0.05$), * ($p \leq 0.05$), ** ($p \leq 0.01$), *** ($p \leq 0.001$). Each experiment is repeated five times.}
  \label{fig:Comparison Tracking Errors}
\end{figure*}

As shown in Fig.~\ref{fig:Tracking Error Box Plot}, the ANOVA results indicate that our method maintains statistically consistent errors across different payloads ($p>0.05$), demonstrating strong robustness, whereas the adaptive method shows significant variations with payload changes ($p<0.05$).

In summary, our approach offers several advantages over the learning-based adaptive control \cite{ahanda2022adaptive}. As shown in Fig.~\ref{fig: Tracking S1}, it achieves comparable tracking performance without a payload, lower steady-state error, and stable performance under varying payloads (Fig.~\ref{fig:Comparison Tracking Errors}). It also exhibits a shorter transient response, as seen in the trajectories and control signals during the first second (Figs.~\ref{fig: Tracking S1} and \ref{fig:Tracking S1 inner command}). Moreover, our design is simpler than a neural network, requiring fewer tunable parameters and offering greater interpretability. Finally, the proposed disturbance compensation supports both tracking and safe control with state constraints, whereas \cite{ahanda2022adaptive} cannot handle constraints.

\subsection{Safe Control}
It is worth noting that $\gamma$ in \eqref{eq:qp} critically influences the CBF’s behavior. A larger $\gamma$ drives the trajectory closer to the safety set boundary (more aggressive), while a smaller $\gamma$ keeps it farther away (more conservative). To achieve a safe yet efficient controller, we set $\gamma = 10$ in all experiments. {We have conducted experiments in three scenarios.}


\subsubsection{{\textbf{Scenario 1: Safety and Robustness Testing with Different Payloads}}}
A virtual boundary is defined at $y_0= \begin{bmatrix} y_1& -0.1& y_3 \end{bmatrix}^T$ as a spatial constraint for the end effector. It moves freely along the $x$ and $z$ axes but is constrained along the $y$ axis with a lower limit of $-0.1$ m. The system is tested with payloads of $1.0$, $1.5$, and $2.0$ kg attached to the end effector.

Three approaches are compared: (1) conventional CBF with a nominal model (see Fig. \ref{fig: CBFNoESOCompare}); (2) our robust CBF with a nominal model and an ESO for disturbance estimation, compared to the DOB-CBF \cite{dacs2022robust} (see Fig. \ref{fig: MFCBFNoErrorBoundCompare}); and (3) our robust CBF with a nominal model, ESO for disturbance estimation, and estimation error bound, compared to the DOB-CBF with estimation error bound (see Fig. \ref{fig: MFCBFErrorBoundCompare}).

\begin{figure}[htbp]
    \centering
    \includegraphics[width=2.9in]{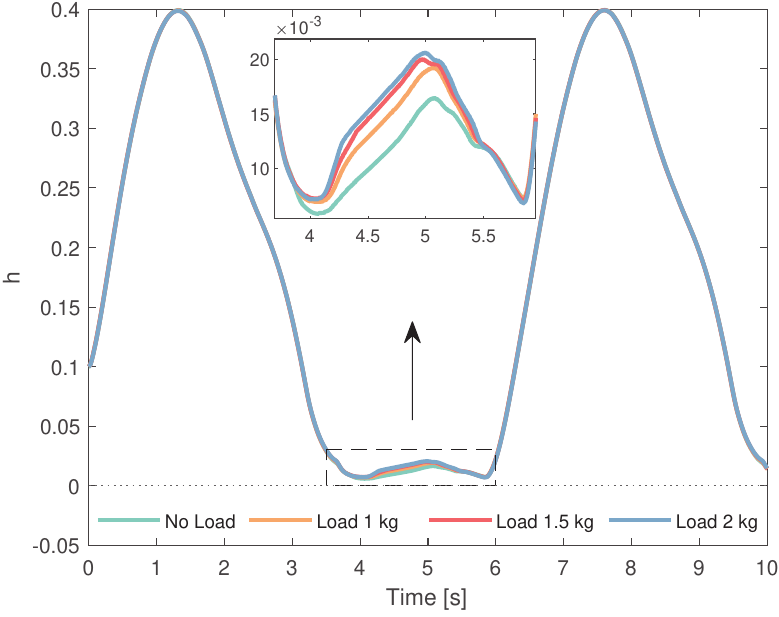}
    \caption{\textbf{Safety boundary: } $\bm{y_0}=\begin{bmatrix}\bm{y_1}&\bm{-0.1}&\bm{y_3}\end{bmatrix}^T$. {The horizontal dashed line at $h = 0$ represents the virtual wall.} Conventional CBF with nominal model: safety performance with variances under different payloads, with the lowest value of $h$ function at $5.9$ mm.}
    \label{fig: CBFNoESOCompare}
\end{figure}

\begin{figure*}[htbp]
\centering
\subfloat[\textnormal{Estimation error bounds not included}]{\includegraphics[width=2.9in]{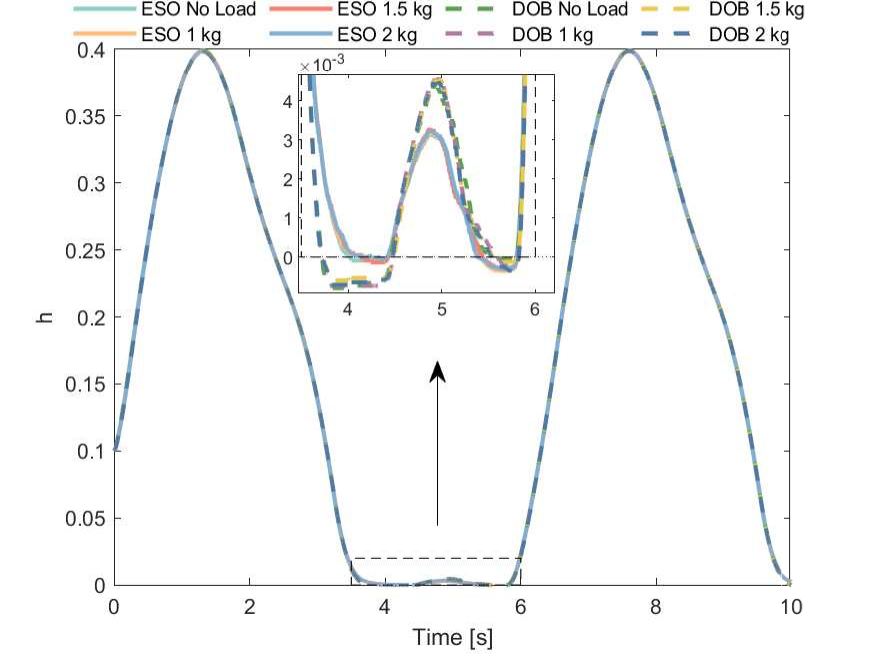}%
\label{fig: MFCBFNoErrorBoundCompare}}
\hfil
\subfloat[\textnormal{Estimation error bounds included}]{\includegraphics[width=2.9in]{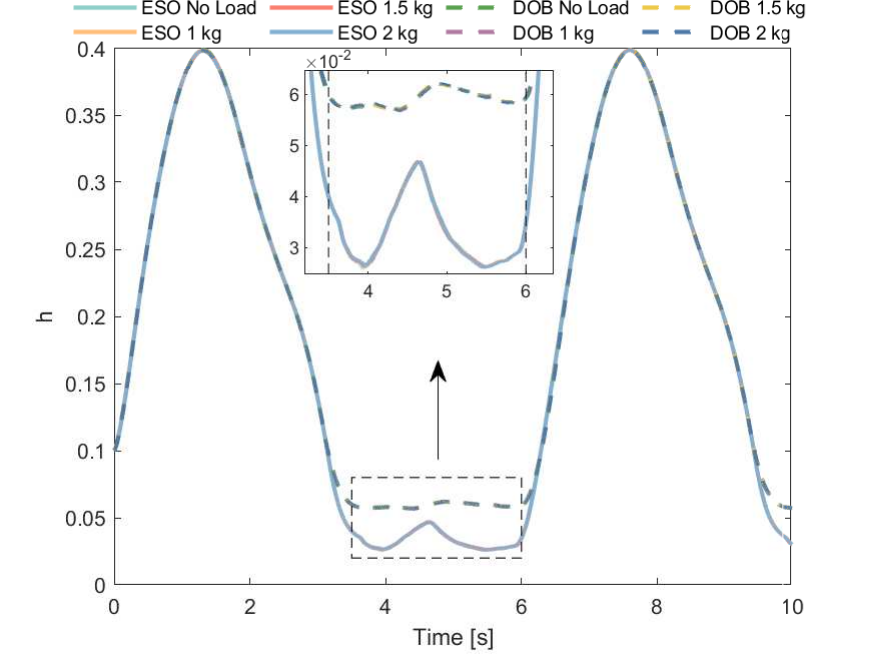}%
\label{fig: MFCBFErrorBoundCompare}}
\caption{{Safety boundary: $\bm{y_0}=\begin{bmatrix}\bm{y_1}&\bm{-0.1}&\bm{y_3}\end{bmatrix}^T$. {The horizontal dashed line at $h = 0$ represents the virtual wall.} (a) CBFs with the nominal model and disturbance compensation by ESO (solid lines) and DOB (dashed lines) show consistent (robust) safety performance under different payloads, with the lowest \( h \) value being \(-0.366\)~mm (unsafe) for our method and \(-0.597\)~mm (unsafe) for the DOB-CBF. (b) CBFs with the nominal model, disturbance compensation by ESO (solid lines) and DOB (dashed lines), and error bounds show consistent (robust) safety performance under different payloads, with the lowest \( h \) value being {26.1~mm (safe and less conservative) for our method and 56.8~mm (safe and more conservative) for the DOB-CBF}.} }
\end{figure*}

Fig. \ref{fig: CBFNoESOCompare} shows the safety performance of the CBF using only the nominal second-order model \eqref{eq:reduced nominal model}, which neglects model errors, inner-loop uncertainty, and external disturbances. {The safety function $h$ varies under different payloads, revealing limited robustness.}

We implement the DOB-CBF \cite{dacs2022robust} for comparison. The DOB-CBF estimates \( b_e(x, d) = \frac{\partial \zeta}{\partial q} f \) in the CBF, while our method directly estimates \( f \) in the original dynamics. Thus, the two observers target different quantities, and their gains are not directly comparable. For fairness, excluding error bounds, we tune both observers’ gains to achieve similar minimum \( h \) magnitudes, indicating comparable safety performance. Fig. \ref{fig: MFCBFNoErrorBoundCompare} compares the $h$ functions of the DOB-CBF and the proposed method, showing that:

\begin{enumerate}
\item The safety function $h$ exhibits consistent performance under different payloads, indicating robustness of both methods.

\item A slight safety violation occurs due to disturbance estimation error, with minimum values of $-0.597$ mm for DOB-CBF and $-0.366$ mm for the proposed method, highlighting the need to account for estimation error bounds in both.
\end{enumerate}

We collect experimental data for \(\hat{b}_e(x,d)\) and \(\hat{f}\) from the DOB and ESO, respectively. Their finite differences yield the maximum rates of change, used as estimation error bounds. Fig. \ref{fig: MFCBFErrorBoundCompare} shows the resulting $h$ values after incorporating these bounds. Compared with Fig. \ref{fig: MFCBFNoErrorBoundCompare}, we observe that:

\begin{enumerate}
\item The trajectories remain safe and consistent across different payloads, demonstrating the robustness of both methods.  
\item Incorporating the estimation error bound raises $h$, keeping trajectories strictly within the safety set (positive $h$ values).  
\item With error bounds, our method stays closer to the safety boundary than \cite{dacs2022robust}, indicating lower conservatism. This may be because \cite{dacs2022robust} estimates \( b_e(x,d) = \frac{\partial \zeta}{\partial q} f \), which includes a state-dependent Jacobian, making it more sensitive to measurement noise and compounding estimation errors.  

\end{enumerate}

{\subsubsection{{\textbf{Scenario 2: Safety Testing Subject to External Disturbance}}}
{In this scenario, a constant 6V is applied to $q_1$, the waist joint, intentionally generating a continuous external perturbation that pushes the manipulator's end-effector toward the virtual wall to evaluate the controller's response. The motor torque is controlled by an analog voltage.}}

\begin{figure}[htbp]
    \centering
    {\includegraphics[width=0.95\linewidth]{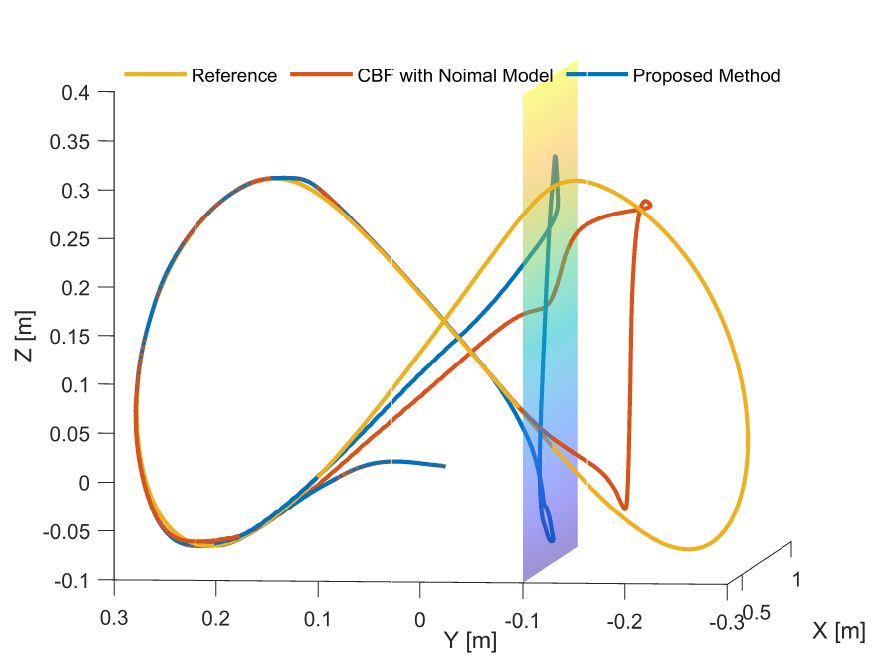}%
    \caption{\textbf{Safety boundary: } $\bm{y_0}=\begin{bmatrix}\bm{y_1}&\bm{-0.1}&\bm{y_3}\end{bmatrix}^T$. {Trajectory comparison of CBF subject to external disturbance. The trajectory (blue) always stays in the safe zone. The unsafe trajectory (red) crosses the safety boundary.}}
    \label{fig:Trajectory with Conventional CBF}}
\end{figure}

{Fig. \ref{fig:Trajectory with Conventional CBF} illustrates a trajectory comparison of CBFs with and without disturbance compensation. The red line shows the unsafe trajectory of the CBF that uses only the nominal model \eqref{eq:reduced nominal model}. While the tracking performance is satisfactory due to the nominal controller’s disturbance rejection capacity, the CBF is expected to intervene when the reference trajectory crosses the safety boundary. Unfortunately, the final trajectory still crosses the safety boundary because the CBF does not account for the external disturbance pushing the end-effector toward the unsafe region. {In contrast, as shown in the blue line,} our approach ensures high-performance tracking when safety is not a concern (see area $y=[0, 0.3]$) and safe control, as the trajectory never crosses the safety boundary.}

{\subsubsection{{\textbf{Scenario 3: Safety Testing Subject to Gravity Disturbance}}}
In this scenario, we set the virtual boundary as $y_0= \begin{bmatrix} y_1& y_2& -0.01 \end{bmatrix}^T$. The end effector is allowed unrestricted motion along the $x$- and $y$-axes, but movement along the $z$-axis is constrained by a lower limit of $-0.01$ m. The disturbance caused by gravity will inevitably push the end effector toward the virtual boundary, leading to unsafe behavior. {Remember that our nominal model \eqref{eq:reduced nominal model}, which has minimal model information, treats gravity as a disturbance.}}

\begin{figure}[htbp]
    \centering
    {\includegraphics[width=0.95\linewidth]{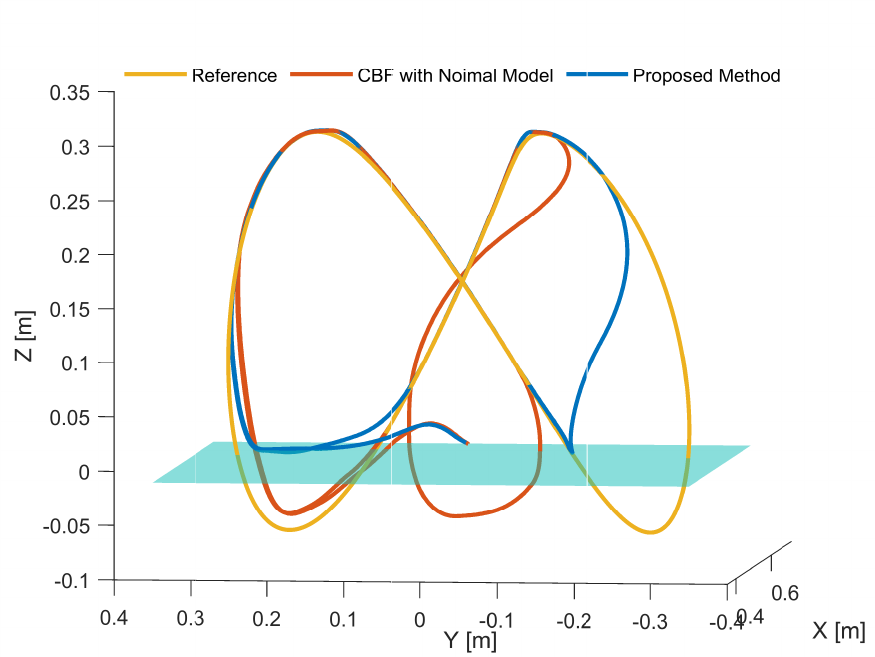}%
    \caption{\textbf{Safety boundary:} $\bm{y_0}=\begin{bmatrix}\bm{y_1}&\bm{y_2}&\bm{-0.01}\end{bmatrix}^T$. {Trajectory comparison of CBFs subject to gravity disturbance. The safe trajectory (blue, ours) remains within the safe zone, while the unsafe trajectory (red, without disturbance compensation) crosses the safety boundary.}
    \label{fig:Conventional CBF z}}}
\end{figure}
{Fig. \ref{fig:Conventional CBF z} shows the trajectory comparison of CBFs subject to gravity disturbance. The red line illustrates the unsafe trajectory of a CBF using only the nominal model \eqref{eq:reduced nominal model}. In contrast, our robust CBF significantly improves performance, as shown by the blue line, which consistently remains within the safe zone.}

\section{Conclusion and future work}
\label{Conslusion and future work}

{This paper presents a unified, plug-and-play outer-loop framework ensuring safe kinematic tracking for closed-architecture robots. The proposed approach achieves robust performance against unknown inner-loop controllers, uncertain dynamics, and external disturbances. Theoretical stability and safety guarantees are validated via real-time experiments on a PUMA 500 robot. Future work will extend this framework to visual servoing systems to address perception-level uncertainties.}

\section{Acknowledgments}
This material is based upon work supported by the National Science Foundation under Grant Nos. 2301543 and 2525200.

\bibliographystyle{unsrtnat}
\bibliography{references}

@article{keemink2018admittance,
  title={Admittance control for physical human--robot interaction},
  author={Keemink, Arvid QL and Van der Kooij, Herman and Stienen, Arno HA},
  journal={The International Journal of Robotics Research},
  volume={37},
  number={11},
  pages={1421--1444},
  year={2018},
  publisher={SAGE Publications Sage UK: London, England}
}

@article{cao2025load,
  title={Load-transfer suspended backpack with bio-inspired vibration isolation for shoulder pressure reduction across diverse terrains},
  author={Cao, Yu and Zhang, Mengshi and Huang, Jian and Mohammed, Samer},
  journal={IEEE Transactions on Robotics},
  year={2025},
  publisher={IEEE}
}

@article{ames2016control,
  title={Control barrier function based quadratic programs for safety critical systems},
  author={Ames, Aaron D and Xu, Xiangru and Grizzle, Jessy W and Tabuada, Paulo},
  journal={IEEE Transactions on Automatic Control},
  volume={62},
  number={8},
  pages={3861--3876},
  year={2017},
  publisher={IEEE}
}

@inproceedings{chen2023robust,
  title={Robust Control Barrier Functions for Safe Control Under Uncertainty Using Extended State Observer and Output Measurement},
  author={Chen, Jinfeng and Gao, Zhiqiang and Lin, Qin},
  booktitle={2023 62nd IEEE Conference on Decision and Control},
  pages={8477--8482},
  year={2023},
  organization={IEEE}
}

@inproceedings{xiao2019control,
  title={Control barrier functions for systems with high relative degree},
  author={Xiao, Wei and Belta, Calin},
  booktitle={2019 IEEE 58th Conference on Decision and Control},
  pages={474--479},
  year={2019},
  organization={IEEE}
}

@article{khalaf2019trajectory,
  title={Trajectory optimization of robots with regenerative drive systems: Numerical and experimental results},
  author={Khalaf, Poya and Richter, Hanz},
  journal={IEEE Transactions on Robotics},
  volume={36},
  number={2},
  pages={501--516},
  year={2019},
  publisher={IEEE}
}

@inproceedings{molnar2023safety,
  title={Safety-critical control with bounded inputs via reduced order models},
  author={Molnar, Tamas G and Ames, Aaron D},
  booktitle={2023 American Control Conference},
  pages={1414--1421},
  year={2023},
  organization={IEEE}
}

@article{xue2015performance,
  title={Performance analysis of active disturbance rejection tracking control for a class of uncertain {LTI} systems},
  author={Xue, Wenchao and Huang, Yi},
  journal={ISA Transactions},
  volume={58},
  pages={133--154},
  year={2015},
  publisher={Elsevier}
}

@article{singletary2021safety,
  title={Safety-critical kinematic control of robotic systems},
  author={Singletary, Andrew and Kolathaya, Shishir and Ames, Aaron D},
  journal={IEEE Control Systems Letters},
  volume={6},
  pages={139--144},
  year={2021},
  publisher={IEEE}
}

@article{liu2022advancements,
  title={Advancements in accuracy decline mechanisms and accuracy retention approaches of {CNC} machine tools: a review},
  author={Liu, Wenjun and Zhang, Song and Lin, Jianghai and Xia, Yuhai and Wang, Jiaxiang and Sun, Yingli},
  journal={The International Journal of Advanced Manufacturing Technology},
  volume={121},
  number={11},
  pages={7087--7115},
  year={2022},
  publisher={Springer}
}

@article{ahanda2022task,
  title={Task-space control for industrial robot manipulators with unknown inner loop control architecture},
  author={Ahanda, Joseph Jean-Baptiste Mvogo and Aba, Charles Medzo and Melingui, Achile and Zobo, Bernard Essimbi and Merzouki, Rochdi},
  journal={Journal of the Franklin Institute},
  volume={359},
  number={12},
  pages={6286--6310},
  year={2022},
  publisher={Elsevier}
}

@article{khan2024control,
  title={Control of robot manipulators with uncertain closed architecture using neural networks},
  author={Khan, Gulam Dastagir},
  journal={Intelligent Service Robotics},
  pages={1--13},
  year={2024},
  publisher={Springer}
}

@article{molnar2021model,
  title={Model-free safety-critical control for robotic systems},
  author={Molnar, Tamas G and Cosner, Ryan K and Singletary, Andrew W and Ubellacker, Wyatt and Ames, Aaron D},
  journal={IEEE Robotics and Automation Letters},
  volume={7},
  number={2},
  pages={944--951},
  year={2021},
  publisher={IEEE}
}

@book{perko2013differential,
  title={Differential equations and dynamical systems},
  author={Perko, Lawrence},
  volume={7},
  year={2013},
  publisher={Springer Science \& Business Media}
}

@article{wang2019dynamic,
  title={Dynamic modularity approach to adaptive control of robotic systems with closed architecture},
  author={Wang, Hanlei and Ren, Wei and Cheah, Chien Chern and Xie, Yongchun and Lyu, Shangke},
  journal={IEEE Transactions on Automatic Control},
  volume={65},
  number={6},
  pages={2760--2767},
  year={2019},
  publisher={IEEE}
}

@inproceedings{dacs2022robust,
  title={Robust safe control synthesis with disturbance observer-based control barrier functions},
  author={Da{\c{s}}, Ersin and Murray, Richard M},
  booktitle={2022 IEEE 61st Conference on Decision and Control},
  pages={5566--5573},
  year={2022},
  organization={IEEE}
}

@inproceedings{ahanda2022adaptive,
  title={Adaptive Tracking Control for Industrial Robot Manipulators with Unknown Inner loop Architecture},
  author={Ahanda, Joseph Jean-Baptiste Mvogo and Melingui, Achille and Lakhal, Othman and Zobo, Bernard Essimbi and Kadri, Hela and Merzouki, Rochdi},
  booktitle={2022 International Conference on Robotics and Automation},
  pages={9860--9866},
  year={2022},
  organization={IEEE}
}

@inproceedings{gao2003scaling,
  title={Scaling and bandwidth-parameterization based controller tuning},
  author={Gao, Zhiqiang},
  booktitle={2003 American Control Conference},
  pages={4989--4996},
  year={2003},
  organization={IEEE}
}

@article{chen2022relationship,
      title={A General Model-Based Extended State Observer with Built-In Zero Dynamics}, 
      author={Jinfeng Chen and Zhiqiang Gao and Yu Hu and Sally Shao},
      year={2023},
      journal={arXiv preprint arXiv:2208.12314},
}

@inproceedings{nguyen2016exponential,
  title={Exponential control barrier functions for enforcing high relative-degree safety-critical constraints},
  author={Nguyen, Quan and Sreenath, Koushil},
  booktitle={2016 American Control Conference},
  pages={322--328},
  year={2016},
  organization={IEEE}
}

@book{blanchini2008set,
  title={Set-theoretic methods in control},
  author={Blanchini, Franco and Miani, Stefano and others},
  volume={78},
  year={2008},
  publisher={Springer}
}

@ARTICLE{cao2024safety,
  author={Cao, Zhongkun and Mao, Jianliang and Zhang, Chuanlin and Cui, Chenggang and Yang, Jun},
  journal={IEEE Transactions on Industrial Electronics}, 
  title={Safety-Critical Generalized Predictive Control for Speed Regulation of {PMSM} Drives Based on Dynamic Robust Control Barrier Function}, 
  year={2025},
  volume={72},
  number={2},
  pages={1881-1891}}

@inproceedings{zhao2020adaptive,
  title={Adaptive robust quadratic programs using control {Lyapunov} and barrier functions},
  author={Zhao, Pan and Mao, Yanbing and Tao, Chuyuan and Hovakimyan, Naira and Wang, Xiaofeng},
  booktitle={2020 59th IEEE Conference on Decision and Control},
  pages={3353--3358},
  year={2020},
  organization={IEEE}
}

@article{zhang2024eso,
  title={{ESO}-based safety-critical control for robotic systems with unmeasured velocity and input delay},
  author={Zhang, Sihua and Zhai, Di-Hua and Lin, Juncheng and Xiong, Yuhan and Xia, Yuanqing and Wei, Minfeng},
  journal={IEEE Transactions on Industrial Electronics},
  volume={71},
  number={10},
  pages={13053--13063},
  year={2024},
  publisher={IEEE}
}

@article{wang2023composite,
  title={A composite control framework of safety satisfaction and uncertainties compensation for constrained time-varying nonlinear {MIMO} systems},
  author={Wang, Haijing and Peng, Jinzhu and Zhang, Fangfang and Wang, Yaonan},
  journal={IEEE Transactions on Systems, Man, and Cybernetics: Systems},
  volume={53},
  number={12},
  pages={7864--7875},
  year={2023},
  publisher={IEEE}
}

@article{alan2021safe,
  title={Safe controller synthesis with tunable input-to-state safe control barrier functions},
  author={Alan, Anil and Taylor, Andrew J and He, Chaozhe R and Orosz, G{\'a}bor and Ames, Aaron D},
  journal={IEEE Control Systems Letters},
  volume={6},
  pages={908--913},
  year={2021},
  publisher={IEEE}
}

@ARTICLE{kolathaya2019input,
  author={Kolathaya, Shishir and Ames, Aaron D.},
  journal={IEEE Control Systems Letters}, 
  title={Input-to-State Safety With Control Barrier Functions}, 
  year={2019},
  volume={3},
  number={1},
  pages={108-113},
  doi={10.1109/LCSYS.2018.2853698}}

@ARTICLE{buch2022robust,
  author={Buch, Jyot and Liao, Shih-Chi and Seiler, Peter},
  journal={IEEE Control Systems Letters}, 
  title={Robust Control Barrier Functions With Sector-Bounded Uncertainties}, 
  year={2022},
  volume={6},
  number={},
  pages={1994-1999},
  doi={10.1109/LCSYS.2021.3136653}}

@ARTICLE{nguyen2022robust,
  author={Nguyen, Quan and Sreenath, Koushil},
  journal={IEEE Transactions on Automatic Control}, 
  title={Robust Safety-Critical Control for Dynamic Robotics}, 
  year={2022},
  volume={67},
  number={3},
  pages={1073-1088},
  doi={10.1109/TAC.2021.3059156}}

@INPROCEEDINGS{wang2023disturbance,
  author={Wang, Yujie and Xu, Xiangru},
  booktitle={2023 American Control Conference}, 
  title={Disturbance Observer-based Robust Control Barrier Functions}, 
  year={2023},
  volume={},
  number={},
  pages={3681-3687},
  doi={10.23919/ACC55779.2023.10156095}}

@ARTICLE{alan2023parameterized,
  author={Alan, Anil and Molnar, Tamas G. and Ames, Aaron D. and Orosz, Gábor},
  journal={IEEE Control Systems Letters}, 
  title={Parameterized Barrier Functions to Guarantee Safety Under Uncertainty}, 
  year={2023},
  volume={7},
  number={},
  pages={2077-2082},
  doi={10.1109/LCSYS.2023.3285188}}

@ARTICLE{alan2023disturbance,
  author={Alan, Anil and Molnar, Tamas G. and Daş, Ersin and Ames, Aaron D. and Orosz, Gábor},
  journal={IEEE Control Systems Letters}, 
  title={Disturbance Observers for Robust Safety-Critical Control With Control Barrier Functions}, 
  year={2023},
  volume={7},
  number={},
  pages={1123-1128},
  doi={10.1109/LCSYS.2022.3232059}}

@ARTICLE{sun2024safety,
  author={Sun, Jiankun and Yang, Jun and Zeng, Zhigang},
  journal={IEEE Transactions on Automatic Control}, 
  title={Safety-Critical Control With Control Barrier Function Based on Disturbance Observer}, 
  year={2024},
  volume={69},
  number={7},
  pages={4750-4756},
  doi={10.1109/TAC.2024.3352707}}

@article{mrdjan2018robust,
title = {Robust control barrier functions for constrained stabilization of nonlinear systems},
journal = {Automatica},
volume = {96},
pages = {359-367},
year = {2018},
issn = {0005-1098},
doi = {https://doi.org/10.1016/j.automatica.2018.07.004},
author = {Mrdjan Jankovic}}

@ARTICLE{wang2024safety,
  author={Wang, Xinming and Yang, Jun and Liu, Cunjia and Yan, Yunda and Li, Shihua},
  journal={IEEE Transactions on Automatic Control (Early Access)}, 
  title={Safety-critical Disturbance Rejection Control of Nonlinear Systems with Unmatched Disturbances}, 
  year={2025},
  volume={},
  number={},
  pages={},
  doi={10.1109/TAC.2024.3496572}}

@ARTICLE{das2025robust,
  author={Daş, Ersin and Burdick, Joel W.},
  journal={IEEE Transactions on Automatic Control (Early Access)}, 
  title={Robust Control Barrier Functions using Uncertainty Estimation with Application to Mobile Robots}, 
  year={2025},
  volume={},
  number={},
  pages={},
  doi={10.1109/TAC.2025.3538742}}

@article{niaz2026twist,
  title={Twist-based constant-speed path-following controller for robot manipulators with path invariance},
  author={Niaz, Hassan and Pagilla, Prabhakar R},
  journal={Mechatronics},
  volume={116},
  pages={103485},
  year={2026},
  publisher={Elsevier}
}

\end{document}